\documentclass[10pt]{article}
\usepackage[utf8]{inputenc}
\usepackage{geometry}
\usepackage{bbm}
\usepackage{amsmath,amssymb,amsthm}
\newtheorem{theorem}{Theorem}
\newtheorem{proposition}[theorem]{Proposition}
\newtheorem*{proposition*}{Proposition}

\newtheorem{corollary}[theorem]{Corollary}
\newtheorem{lemma}[theorem]{Lemma}
\newtheorem{definition}{Definition}
\newtheorem{remark}{Remark}
\newtheorem{conjecture}{Conjecture}

\usepackage{graphicx}
\usepackage{subfig}
\usepackage{float}
\usepackage{algorithm}
\usepackage{algorithmic}
\usepackage[toc,page]{appendix}
\usepackage[table,xcdraw]{xcolor}
\usepackage[subfigure]{tocloft}
\makeatother
\newlength\mylen

\renewcommand\cftpartpresnum{Part~}
\newcommand{\R}{\mathbb{R}}
\newcommand{\E}{\mathbb{E}}
\newcommand{\F}{\mathcal{F}}
\newcommand{\PS}{\mathcal{P}}

\newcommand{\N}{\mathcal{N}}

\newcommand{\D}{\mathcal{D}}
\newcommand{\M}{\mathcal{M}}
\newcommand{\HS}{\mathcal{H}}
\newcommand{\e}{\mathbf{e}}

\newcommand{\w}{\mathbf{w}}
\newcommand{\x}{\mathbf{x}}
\newcommand{\z}{\mathbf{z}}
\newcommand{\y}{\mathbf{y}}
\newcommand{\lb}{\langle}
\newcommand{\rb}{\rangle}
\newcommand{\ep}{\epsilon}

\newcommand{\Lip}{\text{Lip}}

\newcommand{\TV}{\text{TV}}
\newcommand{\sprt}{\text{sprt}}
\usepackage{hyperref}
\allowdisplaybreaks
\usepackage{authblk}

\title{Generalization Error of GAN from the Discriminator's Perspective}
\author[1]{Hongkang Yang\thanks{hongkang@princeton.edu}}
\author[1,2,3]{Weinan E\thanks{weinan@math.princeton.edu}}
\affil[1]{Program in Applied and Computational Mathematics, Princeton University}
\affil[2]{Department of Mathematics, Princeton University}
\affil[3]{Beijing Institute of Big Data Research}

\date{}

\begin{document}

\maketitle

\begin{abstract}
The generative adversarial network (GAN) is a well-known model for learning high-dimensional distributions, but the mechanism for its generalization ability is not understood.
In particular, GAN is vulnerable to the memorization phenomenon, the eventual convergence to the empirical distribution.
We consider a simplified GAN model with the generator replaced by a density, and analyze how the discriminator contributes to generalization.
We show that with early stopping, the generalization error measured by Wasserstein metric escapes from the curse of dimensionality, despite that in the long term, memorization is inevitable.
In addition, we present a hardness of learning result for WGAN.
\end{abstract}
\textbf{Keywords:} probability distribution, generalization error, curse of dimensionality, early stopping, Wasserstein metric.

\section{Introduction}

The task of generative modeling in machine learning is as follows:
Given $n$ sample points of an unknown probability distribution $P_*$,
we would like to approximate $P_*$ well enough to be able to generate new samples.
The generative adversarial network (GAN) \cite{goodfellow2014generative} is among the most popular models for this task.
It has found diverse applications such as image generation \cite{karras2019style}, photo editing \cite{wu2019gp} and style transfer \cite{isola2017image}.
More importantly, there are emerging scientific applications including inverse problems \cite{shah2018solving}, drug discovery \cite{prykhodko2019molecular}, cosmological simulation \cite{mustafa2019cosmogan}, material design \cite{mao2020material} and medical image coloration \cite{liang2021unpaired}, to name a few.

Despite these promising successes, we are far from a satisfactory theory.
Arguably, one of the most important problems is the generalization ability of GANs,
namely how they are able to estimate the underlying distributions well enough to be able to generate new samples that appear highly realistic.
There are at least two difficulties with generalization:
\begin{enumerate}
\item Curse of dimensionality:

Let $P_*^{(n)}$ be the empirical distribution associated with the $n$ given sample points.
Let $W_2$ be the Wasserstein metric.
It is known that for any absolutely continuous $P_*$ \cite{weed2017sharp}
\begin{equation}
\label{W2 CoD}
W_2(P_*,P_*^{(n)}) ~\gtrsim~ n^{-\frac{1}{d-\delta}}
\end{equation}
for any $\delta > 0$.
The significance of (\ref{W2 CoD}) is that it sets a lower bound for the generalization error of all possible models:
Let $A$ be any algorithm (a mapping) that maps from an $n$ sample set $X_n = \{\x_1,\dots\x_n\}$ of $P_*$ to an estimated distribution $A(X_n)$, then \cite{singh2018minimax}
\begin{equation*}
\inf_{A} \sup_{P_*} ~\E_{X_n}\big[W_2^2\big(P_*,A(X_n)\big)\big] ~\gtrsim ~n^{-\frac{2}{d}}
\end{equation*}
where $P_*$ ranges among all distributions supported in $[0,1]^d$.
Thus, an exponential amount of samples $\epsilon^{-\Omega(d)}$ is needed to achieve an error of $\epsilon$, which is formidable for high-dimensional tasks (e.g. $d \geq 10^5$ for images).
To overcome the curse of dimensionality, one needs to restrict to a smaller class of target distributions $\mathcal{P}_*$ such that
\begin{equation}
\label{hypothesis space}
\inf_{A} \sup_{P_* \in \mathcal{P}_*} ~\E_{X_n}\big[W_2^2\big(P_*,A(X_n)\big)\big] ~\lesssim ~n^{-\alpha}
\end{equation}
for some constant $\alpha > 0$.
Of course, to be relevant in practice, this restricted space of target distributions should not be too small.

\item Memorization during training:

There is another complication that we cannot avoid.
It was argued in \cite{goodfellow2014generative} that the true optimizer of the GAN model is simply $P_*^{(n)}$.
This is disappointing since $P_*^{(n)}$ does not provide any new information beyond the data we already know.
It suggests that if a training algorithm converges, most likely it converges to $P_*^{(n)}$, i.e. it simply memorizes the 
samples provided.
This is known as the ``memorization effect'' and has been analyzed on a simpler model in \cite{yang2020generalization}.
Our best hope is there are intermediate times during the training process when the model provides a more useful approximation to the target distribution.
However, this is a rather subtle issue.
For instance, if we train a distribution $P_t$ by Wasserstein gradient flow over the loss
\begin{equation*}
L(P) = W_2^2(P, P_*^{(n)})
\end{equation*}
then the training trajectory is exactly the Wasserstein geodesic from the initialization $P_0$ to $P_*^{(n)}$.
Since the $W_2$ space has positive curvature (Theorem 7.3.2 \cite{ambrosio2008gradient}),
the pairwise distances are comparable to those on a fat triangle (Figure \ref{fig: W2 geodesic} curve \textcircled{1}).
This suggests that the generalization error $W_2(P_t,P_*)$ along the entire trajectory will be $n^{-O(1/d)}$ and suffers from the curse of dimensionality.

\item Mode collapse and mode dropping.

These are the additional difficulties that are often encountered in the training of GAN models.
\end{enumerate}

\begin{figure}[h]
\centering
\includegraphics[scale=0.2]{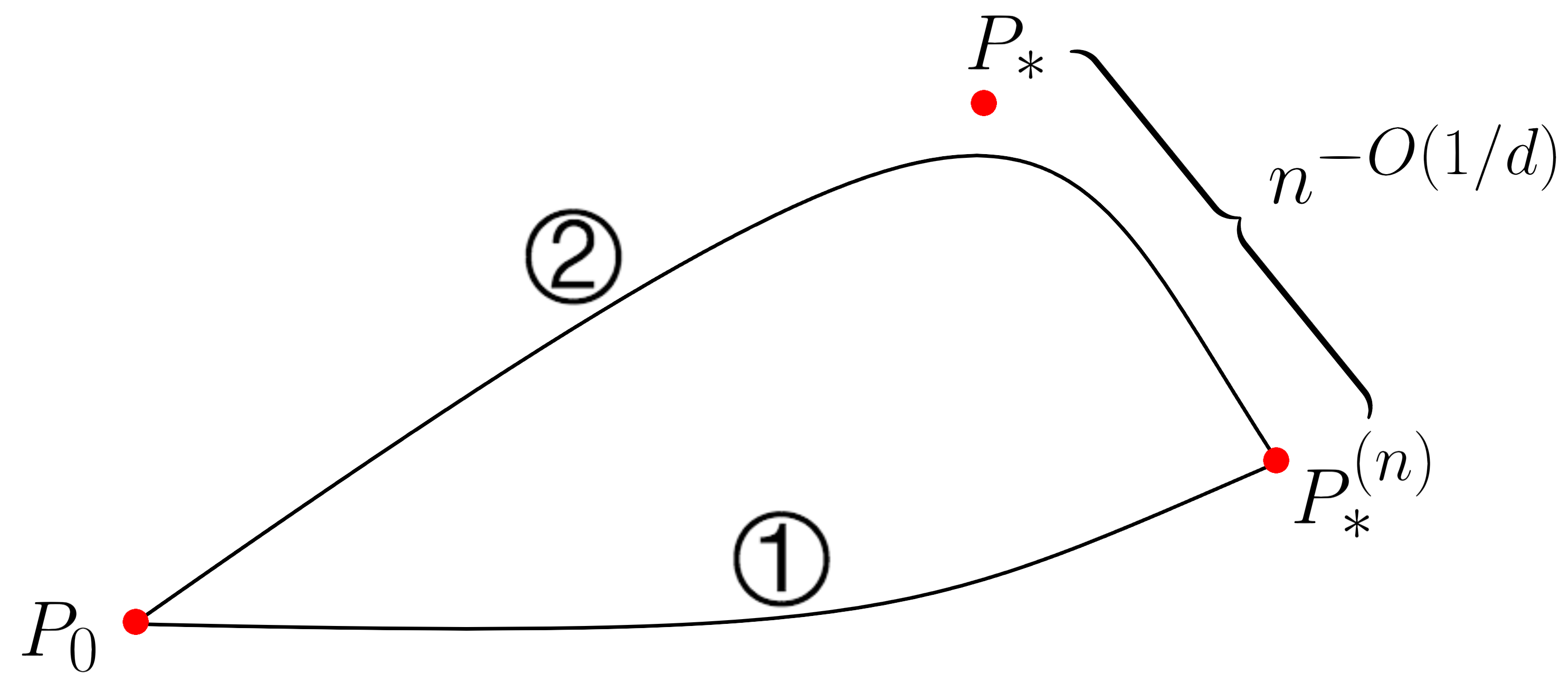}
\caption{Curve \textcircled{1} is the $W_2$ geodesic connecting $P_0$ and $P_*^{(n)}$.
Curve \textcircled{2} is a training trajectory that generalizes.}
\label{fig: W2 geodesic}
\end{figure}

Recently, an approach was proposed by \cite{yang2020generalization} to establish the generalization ability of a density estimator (the bias-potential model).
Its hypothesis space $\mathcal{P}_*$ consists of Boltzmann distributions generated by kernel functions and satisfies universal approximation.
Also, its training trajectory generalizes well:
As illustrated by curve \textcircled{2} in Figure \ref{fig: W2 geodesic}, the trajectory comes very close to $P_*$ before eventually turning toward memorizing $P_*^{(n)}$.
Thus, the generalization error achieved by early-stopping escapes from the curse of dimensionality:
\begin{equation*}
\text{KL}(P_*\|P) = O(n^{-1/4})
\end{equation*}
This generalization ability is enabled by the function representation of the model.
If the function class has small Rademacher complexity, then the model is insensitive to the sampling error $P_*-P_*^{(n)}$, and thus the generalization gap emerges very slowly.

The implication to GAN is that its generalization ability should be attributable to (the function representation of) both its generator and discriminator.
This paper will focus on analyzing the discriminator.  In place of the generator, we will deal directly with the probability distribution.
Our result confirms the intuition of Figure \ref{fig: W2 geodesic} curve \textcircled{2}:
Despite the eventual memorization, early-stopping achieves a generalization error of
\begin{equation}
\label{W2 escapes CoD}
W_2(P_*,P) = O(n^{-\alpha})
\end{equation}
with a constant exponent $\alpha>0$.

This paper is structured as follows.
Section \ref{sec. problem setting} introduces our toy model of GAN.
Sections \ref{sec. two time scale} and \ref{sec. one time scale} present the main results of this paper on generalization error (\ref{W2 escapes CoD}) with 
a two-time-scale training model ($\alpha=1/6$) and a one-time-scale training model ($\alpha=1/8$) respectively.
We also show that memorization always happens in the long time limit.
Section \ref{sec. slow deterioration} provides a supplementary result, that it is intractable to learn the Lipschitz discriminator of WGAN.
Section \ref{sec. proofs} contains all the proofs. Section \ref{sec. discussion} concludes this paper
with remarks on future directions.

\subsection{Related work}

\begin{itemize}
\item GAN training:

Currently there is little understanding of the training dynamics and convergence property of GAN, or any distribution learning model with a generator such as the variational autoencoder \cite{kingma2013auto} and normalizing flows \cite{tabak2010density}.
The available results deal with either simplified models \cite{mescheder2018training,wu2019onelayer,feizi2020LQG,balaji2021understanding} or convergence to local saddle points \cite{nagarajan2017gradient,heusel2017gans,lin2020gradient}.
The situation is further complicated by training failures such as mode collapse \cite{arora2018GAN,che2016mode}, mode dropping \cite{yazici2020empirical} and oscillation \cite{radford2015unsupervised,chavdarova2018sgan}.
Since our emphasis is on the discriminator, we omit the generator to make the analysis tractable.

\item GAN generalization:

Due to the difficulty of analyzing the convergence of GAN, generalization error estimates have been obtained only for simplified models whose generators are linear maps or their variants \cite{feizi2020LQG,wu2019onelayer,lei2020sgd}.
Another line of works \cite{arora2017generalization,zhang2017discrimination,bai2019approximability} focuses on the ``neural network distance", which are the GAN losses defined by neural network discriminators, and shows that these distances have a sampling error of $O(n^{-1/2})$.
Although \cite{arora2017generalization} interprets this result as the inability of the discriminator to detect a lack of diversity, we show that this is in fact an advantage, such that the discriminator enables the training process to generalize well.
Meanwhile, \cite{gulrajani2020towards} proposes to use the neural network distance with held-out data to measure memorization, while \cite{nagarajan2018memorization} discussed the dependence of memorization on the number of latent samples drawn by GAN.

\item GAN design and regularization:

The improvement of GAN has mainly focused on three aspects: alternative loss functions (e.g. WGAN \cite{arjovsky2017wasserstein} and least-square GAN \cite{mao2018effectiveness}), novel function representations (e.g. fully convolutional \cite{radford2015unsupervised} and self-attention \cite{jiang2021transgan}), and new regularizations.
There are roughly three kinds of regularizations: the regularizations on the function values (e.g. gradient penalty \cite{gulrajani2017improved} and $L^2$ penalty \cite{xu2020understanding}), the regularizations on the parameters (e.g. spectral normalization \cite{miyato2018spectral} and weight decay \cite{krogh1992simple}) and the regularizations on the input values (e.g. batch normalization \cite{ioffe2015batch} and layer normalization \cite{ba2016layer}).
See \cite{saxena2021generative} for a comprehensive review.
Our proofs indicate how function representation and regularization influence the generalization ability of GAN.

\item Function representation:

The function class is central to the theoretical analysis of machine learning models.
A good function representation, such as the Barron space and flow-induced space \cite{e2021barron} (which capture 2-layer networks and residual networks), is the key to the generalization ability of supervised learning models \cite{e2018priori,e2019residual,e2019min} and density estimator \cite{yang2020generalization}.
Broadly speaking, a supervised learning model can be studied as a continuous calculus of variations problem \cite{e2020machine,e2020NNML} determined by four factors: its function representation, loss function, training rule, and their discretizations.
Distribution learning has the additional factor of distribution representation \cite{yang2020generalization}, namely how the probability distribution is represented by functions.
\end{itemize}

\section{Problem Setting}
\label{sec. problem setting}

Consider the domain $\Omega = [0,1]^d$.
Denote the space of probability measures by $\PS(\Omega)$, and the space of finite signed Radon measures by $\M(\Omega)$.
Denote the $p$-Wasserstein metric on $\PS(\Omega)$ by $W_p$ \cite{villani2003topics}.

Denote the modeled distribution by $P$, the target distribution by $P_*$.
Let $\{\x_i\}_{i=1}^n$ be $n$ i.i.d. samples from $P_*$ and denote the empirical distribution by $P_*^{(n)} = \frac{1}{n}\sum_{i=1}^n \delta_{\x_i}$.

If $P,P_*$ are absolutely continuous, denote their density functions by $p,p_*$.
Conversely, given a density function $p$, denote the corresponding measure by $P$ (i.e. $p$ times the Lebesgue measure on $\Omega$).
For convenience, we use density $p$ and measure $P$ interchangeably when there is no confusion.

We use $f \lesssim g$ or $f = O(g)$ to indicate that $\limsup_{x\to\infty} f(x)/g(x) < \infty$.
We use $f = o(g)$ to indicate that $\lim_{x\to\infty} f(x)/g(x) = 0$.
The notations $f \gtrsim g$ and $f=\Omega(g)$ and $\omega(g)$ are defined similarly.
We use $f \asymp g$ to indicate that $f \lesssim g$ and $f \gtrsim g$.

{\color{black}
\subsection{Distribution Representation}
\label{sec. distribution representation}

As we have discussed in the introduction, the generalization ability of a model comes from the representation or parametrization of its component functions.
Thus, for the GAN model, it is reasonable to conjecture that \textit{either} the generator \textit{or} the discriminator alone can enable GAN to generalize well.

Since this paper focuses on the discriminator, we prove that a discriminator with a good parametrization is sufficient for good generalization, without the help of the generator.
To eliminate the confounding effect of the generator, one needs to construct a model whose generated distribution $P$ does not have any parametrization.
The simplest way is to model $P$ by its density function $p$.
We will consider $p$ as a function in $L^2(\Omega)$, and denote the space of probability densities by
\begin{equation*}
\Delta = \big\{p \in L^2(\Omega) ~\big|~ p \geq 0 ~\text{a.e.}, ~\int p = 1 \big\}
\end{equation*}
Henceforth, this abstract model will be referred to as adversarial density estimation.

We will discuss how our results can be applied to the ordinary GAN in Section \ref{sec. the generator}.
}

\subsection{Function representation}
\label{sec. function representation}
To model the discriminator $D$, let us consider a simple function class that captures the two key properties of neural networks: universal approximation and dimension-independent complexity.

Specifically, we model $D$ by the random feature functions (or kernel functions) \cite{rahimi2008uniform,e2021barron}:
\begin{equation}
\label{RFM discriminator}
D(\x) = \E_{\rho_0(\w,b)}[a(\w,b)\sigma(\w\cdot\x+b)]
\end{equation}
where $\sigma$ is the ReLU activation, $\rho_0$ is a fixed parameter distribution, and $a$ is the parameter function to be learned.
Assume for convenience that $\rho_0$ has bounded support: $\sprt \rho_0 \subseteq \{\|\w\|_1+|b| \leq 1\}$.

One can define the kernel
\begin{equation}
\label{RKHS kernel}
k(\x,\x') = \E_{\rho_0}[\sigma(\w\cdot\x+b)\sigma(\w\cdot\x'+b)]
\end{equation}
The RKHS space $\HS$ with kernel $k$ is generated by the norm
\begin{equation}
\label{RKHS norm}
\|D\|_{\HS} = \|a\|_{L^2(\rho_0)}
\end{equation}

Regarding the universal approximation property, assume that $\sprt \rho_0$ contains all directions: For any $(\w,b) \neq 0$, there exists $\lambda > 0$ such that $\lambda (\w,b) \in \sprt \rho_0$ (e.g. let $\rho_0$ be the uniform distribution over the $L^1$ sphere).
Then, \cite{sun2018RFM} implies that the RKHS space $\HS$ is dense in $C(\Omega)$ under the supremum norm.
It follows that the kernel (\ref{RKHS kernel}) is positive definite.

Regarding complexity, the Rademacher complexity of $\HS$ escapes from the curse of dimensionality (Theorem 6 of \cite{e2021barron}).
Specifically, the following bound holds uniformly over any collection of $n$ points $\{\x_i\}_{i=1}^n \subseteq [-1,1]^d$,
\begin{equation*}
Rad_n(\{\|D\|_{\HS}\leq r\}) \leq 2r\frac{\sqrt{2\log 2d}}{\sqrt{n}}
\end{equation*}
{\color{black}This property eventually leads to our generalization error estimates.}

\subsection{Training loss}
Denote by $L(P)$ a loss over the modeled distribution $P$.
The GAN losses are constructed as dual norms over some family $\D$ of discriminators.

The straightforward construction is the WGAN loss \cite{arjovsky2017wasserstein},
\begin{equation}
\label{WGAN loss}
L(P) = \sup_{D \in \D} \E_{P_*}\big[D(\x)\big] - \E_P \big[D(\x)\big] - R(D)
\end{equation}
where $R(D)$ is some regularization on $D$ such as gradient penalty \cite{gulrajani2017improved,kodali2017convergence}.
The classical GAN loss \cite{goodfellow2014generative} is a weak version of the dual formulation of Jensen-Shannon divergence
\begin{equation*}
L(P) = \sup_{D \in \D} \E_{P_*}\big[\log\frac{e^{D(\x)}}{1+e^{D(\x)}}\big] + \E_P \big[\log\frac{1}{1+e^{D(\x)}}\big]
\end{equation*}
There are many other constructions such as the $f$-divergence GAN loss \cite{nowozin2016f}, energy-based GAN loss \cite{zhao2016energy}, least-square GAN loss \cite{mao2018effectiveness} etc.


\subsection{Training Rule}
Rewrite the GAN loss (e.g. formula (\ref{WGAN loss})) as a joint loss in $p$ and $D$
\begin{equation*}
\min_p \max_D L(p, D)
\end{equation*}
To solve this min-max problem, we need to consider the relative time scales for the training of the variables $p$ and $D$:
\begin{enumerate}
\item One time scale training:
The learning rates for the parameters of $D_t$ and for $p_t$ have the same magnitude.

\item Two time scale training:
The learning rate for $D_t$ is much larger than that of $p_t$.
Thus, on the time scale of $p_t$, the discriminator $D_t$ can be assumed optimal (at least when $L(p,D)$ is concave in the parameters of $D$).
Effectively, $p_t$ is trained by gradient descent on $\sup_D L(p,D)$.
This dynamics has been shown to closely approximate two-time scale training \cite{borkar1997timescale,heusel2017gans,lin2020gradient}.
\end{enumerate}

For the specific training rule, we will use continuous-time gradient flow.
For the discriminator implemented by random feature functions (\ref{RFM discriminator}), we train its parameter function $a$ by gradient ascent
\begin{align*}
\frac{d}{dt} a_t(\w,b) = \frac{\delta L(p,D_t)}{\delta a} = \int_{\Omega} \frac{\delta L(p,D_t)}{\delta D}(\x) \sigma(\w\cdot\x+b)
\end{align*}
where the variational gradients are taken in $L^2(\rho_0)$ and $L^2(\Omega)$.
It follows that $D_t$ evolves by
\begin{equation}
\label{D gradient ascent}
\frac{d}{dt}D_t = \E_{\rho_0(\w,b)}\Big[\frac{d}{dt} a_t(\w,b) \sigma(\w\cdot\x+b)\Big] = k*\frac{\delta L(p,D_t)}{\delta D}
\end{equation}
{\color{black}where $k$ is the kernel defined in (\ref{RKHS kernel})} and $k*$ denotes the convolution over $L^2(\Omega)$
\begin{equation}
\label{k convolution}
k*f(\x) = \int_{\Omega} k(\x,\x')f(\x')
\end{equation}

Meanwhile, for the density function $p_t$, one option is the plain gradient descent
\begin{equation}
\label{p gradient descent}
\frac{d}{dt} p_t = -\frac{\delta L(p_t, D_t)}{\delta p}
\end{equation}
In this case, $p_t$ is not guaranteed to remain as a probability density, but becomes a signed measure in $\M(\Omega)$.

An alternative option is to perform projected gradient descent.
Denote the tangent cone of the probability simplex $\Delta$ by
\begin{equation*}
\forall p \in \Delta, \quad T_p\Delta = \{q = q_+-q_- ~|~ q_{\pm} \geq 0, ~q_{\pm} \in L^2(\Omega), ~q_- \ll p\}
\end{equation*}
Let $\Pi_{\Delta}: L^2(\Omega) \to \Delta$ be the $L^2$ projection onto $\Delta$,
and let $\Pi_{T_p\Delta}$ be the projection onto $T_p\Delta$.
Then, the projected flow is given by
\begin{align}
\label{project gradient flow}
\begin{split}
\frac{d}{dt} p_t &= \Pi_{T_{p_t}\Delta} \Big(-\frac{\delta L(p_t,D_t)}{\delta p}\Big) = \lim_{\epsilon\to 0^+} \frac{\Pi_{\Delta}\big(p_t-\epsilon\delta_{p}L(p_t,D_t)\big)-p_t}{\epsilon}
\end{split}
\end{align}

\subsection{Test Loss}
The test loss is set to be the Wasserstein metric $W_2$ between the modeled distribution $P$ and the target $P_*$.
If $p\in L^2(\Omega)$, we consider its projection
\begin{equation}
\label{W2 loss}
W_2\big(\Pi_{\Delta}(p), P_*\big)
\end{equation}

{\color{black}
The $W_2$ metric is chosen for two reasons.
First, a key advantage of $W_2$ is that it is sensitive to memorization, such that any solution that approximates $P_*^{(n)}$ will exhibit the curse of dimensionality (\ref{W2 CoD}).
Thus, a natural criterion for generalization ability is that a good model should achieve a $W_2$ test error with a dimension-independent rate (\ref{W2 escapes CoD}), despite that it is trained using only $P_*^{(n)}$.

Second, the $W_2$ metric can be seen as the $L^2$ regression loss for probability measures, and thus is a natural choice for the loss function.
Specifically, it is the quotient metric on $\PS(\Omega)$ derived from the $L^2$ metric on the generators of GAN.

\begin{proposition*}
[Informal version of Proposition \ref{prop. W2 matching}]
For any target distribution $P_*$ and a distribution $P$ generated by any generator $G$,
\begin{equation*}
W_2(P,P_*) = \inf_{G_*} \|G-G_*\|_{L^2}
\end{equation*}
where $G_*$ is any generator that generates $P_*$.
\end{proposition*}
The details are given in Section \ref{sec. between W2 and L2}.
}

\section{Two Time Scale Training}
\label{sec. two time scale}

First, we consider the setting with explicit regularization on the discriminator.
Section \ref{sec. generalization gap} analyzes the generalization gap, and Section \ref{sec. generalization error} analyzes the generalization error and early stopping.

The training loss is set to be the WGAN loss (\ref{WGAN loss}).
Instead of fixing the family $\D$, we simply penalize the RKHS norm (\ref{RKHS norm}).
Then, the loss associated with the discriminator in terms of the parameter function $a$ becomes
\begin{align}
\label{WGAN + RKHS}
\begin{split}
\max_a L_D(a) &= \E_{P_*}[D] - \E_P[D] - \|D\|_{\HS}^2\\
&= \int \E_{\rho_0(\w,b)} [a(\w,b)\sigma(\w\cdot\x+b)] ~d(P_*-P)(\x) - \|a\|_{L^2(\rho_0)}^2
\end{split}
\end{align}
This loss is strongly concave with a unique maximizer $a_*${\color{black}, obtainable by taking the variational derivative in $a$}
\begin{align*}
a_*(\w,b) &= \frac{1}{2} \int \sigma(\w\cdot\x'+b) ~d(P_*-P)(\x')
\end{align*}
So the optimal discriminator $D_*$ is given by
\begin{align*}
D_*(\x) &= {\color{black}\E_{\rho_0(\w,b)}[a_*(\w,b) \sigma(\w\cdot\x+b)]} = \frac{1}{2} \int k(\x,\x') ~d(P_*-P)(\x')
\end{align*}
or $D_* = \frac{1}{2} k*(P_*-P)$,
where the kernel $k$ is defined by (\ref{RKHS kernel}, \ref{k convolution}).
With two-time-scale training and strong concavity, we can assume that the discriminator is always the maximizer $D_*$, so the WGAN loss (\ref{WGAN loss}) becomes
\begin{equation}
\label{MMD loss}
L(P) = \E_{P_*}[D_*]-\E_P[D_*] = \frac{1}{2} \iint k ~d(P_*-P)^2
\end{equation}
which is the squared MMD metric \cite{gretton2021kernel} with kernel $k$.
Similarly, we denote the empirical loss by
\begin{equation*}
L^{(n)}(P) = \frac{1}{2} \iint k ~d(P_*^{(n)}-P)^2
\end{equation*}

\subsection{Generalization Gap}
\label{sec. generalization gap}

The gradient descent training rule (\ref{p gradient descent}) and projected gradient descent (\ref{project gradient flow}) can be written respectively as
\begin{align}
\label{MMD gradient descent}
\frac{d}{dt} p_t &= k*(P_*-P_t)\\
\label{MMD projected gradient descent}
\frac{d}{dt} p_t &= \Pi_{T_{p_t}\Delta}\big(k*(P_*-P_t)\big)
\end{align}
If the empirical loss $L^{(n)}$ is used, we denote the training trajectory by $p_t^{(n)}$.

\begin{proposition}[Generalization gap]
\label{prop. generalization gap}
With any target distribution $P_* \in \PS(\Omega)$ and with probability $1-\delta$ over the sampling of $P_*^{(n)}$,
\begin{enumerate}
\item If $p_t, p_t^{(n)}$ are trained by gradient descent (\ref{MMD gradient descent}), then
\begin{equation*}
W_2\big(\Pi_{\Delta}(p_t),\Pi_{\Delta}(p_t^{(n)})\big) \leq \sqrt{d} ~\frac{4\sqrt{2\log 2d} + \sqrt{2\log (2/\delta)}}{\sqrt{n}}t
\end{equation*}
\item If $p_t, p_t^{(n)}$ are trained by projected gradient descent (\ref{MMD projected gradient descent}), then
\begin{equation*}
W_2(P_t,P_t^{(n)}) \leq \sqrt{d} ~\frac{4\sqrt{2\log 2d} + \sqrt{2\log (2/\delta)}}{\sqrt{n}}t
\end{equation*}
\end{enumerate}
\end{proposition}

\subsection{Generalization Error and Early-Stopping}
\label{sec. generalization error}

Let $p_t^{(n)}$ be trained by gradient descent (\ref{MMD gradient descent}) on the empirical loss $L^{(n)}$.

\begin{theorem}[Generalization error]
\label{thm. generalization error}
Given any target density function $p_*$, with probability $1-\delta$ over the sampling of $P_*^{(n)}$,
the generalization error of the trajectory $p_t^{(n)}$ is bounded by
\begin{equation*}
W_2\big(P_*,\Pi_{\Delta}(p_t^{(n)})\big) \leq \sqrt{d}\frac{\|p_*-p_0\|_{\HS}}{\sqrt{t}} + \sqrt{d}~ \frac{4\sqrt{2\log 2d} + \sqrt{2\log (2/\delta)}}{\sqrt{n}}t
\end{equation*}
\end{theorem}

This is a decomposition of the generalization error into training error plus generalization gap.
It follows that with early stopping, we can escape from the curse of dimensionality.

\begin{corollary}[Early stopping]
\label{cor. early stopping}
If we choose an early-stopping time $T$ as follows
\begin{equation*}
T \asymp \|p_*-p_0\|_{\mathcal{H}}^{2/3} \big(\frac{n}{\log d}\big)^{1/3}
\end{equation*}
then the generalization error obeys
\begin{align*}
W_2\big(P_*, \Pi_{\Delta}(p_T^{(n)})\big) &\lesssim \sqrt{d} \|p_*-p_0\|_{\mathcal{H}}^{2/3} \Big(\frac{\log d}{n}\Big)^{1/6}
\end{align*}
\end{corollary}

This result suggests that for the adversarial density estimation model, a polynomial amount of samples $n = O(\epsilon^{-6})$ is needed to achieve an error of $\epsilon$, instead of an exponential amount $\epsilon^{-\Omega(d)}$.

\subsection{Memorization}
\label{sec. memorization}

Despite that early-stopping solutions perform well, this adversarial density estimation model eventually memorizes the samples.

\begin{proposition}[Memorization]
\label{prop. memorization}
Given the condition of Theorem \ref{thm. generalization error}, $P_t^{(n)}$ converges weakly to $P_*^{(n)}$.
\end{proposition}

We show a stronger result in Lemma \ref{lemma. universal convergence} that this model can be trained to converge to any distribution.

\subsection{Remarks}

\begin{remark}[Outside the Hypothesis Space]
\normalfont
Theorem \ref{thm. generalization error} requires that the target density belongs to the hypothesis space
\begin{equation*}
\PS_* = \{p_* \in \Delta~|~\|p_*-p_0\|_{\HS} < \infty\}
\end{equation*}
It is a weak condition, because given any $p_* \in \Delta$, the set $p_* + \HS$ is dense in $L^2(\Omega)$, so there are plenty of initializations $p_0$ that can work.
Even if $p_*-p_0 \notin \HS$, it is straightforward to show that,
\begin{equation*}
\|p_*-p_t\|_{L^2} \leq \inf_{q} \frac{\|q-p_0\|_{\HS}}{\sqrt{t}} + \|p_*-q\|_{L^2}
\end{equation*}
So if the target $p_*$ satisfies
\begin{equation*}
\inf_{\|p_0-q\|_{\HS} \leq R} \|p_*-q\|_{L^2} \lesssim R^{-\beta}
\end{equation*}
for some $\beta > 0$,
then the training error can be bounded by
\begin{equation*}
\|p_*-p_t\|_{L^2} \lesssim t^{-\frac{\beta}{2(1+\beta)}}
\end{equation*}
and the early-stopping generalization error becomes $O(n^{-\frac{\beta}{6\beta+4}})$.
\end{remark}

{\color{black}
\begin{remark}[Finite neurons and approximation error]
\label{remark: finite neurons}
\normalfont
The discriminator (\ref{RFM discriminator}) is defined as an average of a possibly infinite collection of feature functions, but in practice, there is only a finite number $m$ of neurons:
\begin{equation*}
D^{(m)}(\x) = \frac{1}{m} \sum_{j=1}^m a_j \sigma(\w_j\cdot\x+b_j)
\end{equation*}
where each $(\w_j,b_j)$ is sampled from the parameter distribution $\rho_0$.
The model and its training dynamics (\ref{WGAN + RKHS}, \ref{MMD loss}, \ref{MMD gradient descent}) can be adapted to this finite-neuron setting, and we denote the empirical training trajectory by $p^{(n,m)}_t$.
The generalization bound of Theorem \ref{thm. generalization error} continues to hold with an additional term of the approximation error, which scales as $O(\sqrt{t/m})$.
Specifically, with probability $1-2\delta$ over the sampling of $P_*^{(n)}$ and $\rho_0^{(m)}$,
\begin{equation*}
\frac{W_2\big(P_*,\Pi_{\Delta}(p_t^{(n,m)})\big)}{\sqrt{d}} \leq \frac{\|p_*-p_0\|_{\HS}}{\sqrt{t}} + \frac{\|p_*-p_0\|_{\HS} \big[4+\sqrt{2\log(4/\delta)}\big]}{\sqrt{m}} \sqrt{t} + \frac{4\sqrt{2\log 2d} + \sqrt{2\log (2/\delta)}}{\sqrt{n}}t
\end{equation*}
The proof is given in Section \ref{sec. finite neurons}.
In particular, this finite-neuron model is still able to avoid the curse of dimensionality.
\end{remark}
}

\section{One Time Scale Training}
\label{sec. one time scale}

The previous section demonstrates that, with an explicit regularization $\|D\|_{\HS}$, we can bound the Rademacher complexity of the discriminators and enable the adversarial density estimator to generalize well.
This section shows that even if we do not explicitly bound the complexity of $D$, the early-stopping solutions still enjoy good generalization accuracy.

\subsection{Ill-posedness}
Recall that we are using the WGAN loss (\ref{WGAN loss}), which we write as a min-max problem:
\begin{equation*}
\min_p \max_D L(p, D) = \min_p \max_D \E_{P_*}[D] -\E_P[D] - R(D)
\end{equation*}
Instead of penalizing the parameters of the discriminator $R(D)=\|D\|_{\HS}^2=\|a\|_{L^2(\rho_0)}^2$, we consider weaker regularizations on its function value.
For instance,
\begin{itemize}
\item The $L^2$ penalty proposed by \cite{xu2020understanding}
\begin{equation}
\label{L2 penalty}
R(D) = \|D\|_{L^2(\Omega)}^2
\end{equation}
\item Gradient penalty
\begin{equation}
\label{gradient penalty}
R(D) = \|\nabla D\|_{L^2(\Omega)}^2
\end{equation}
\item The Lipschitz penalty proposed by \cite{gulrajani2017improved}.
We present a simplified form for better illustration
\begin{equation}
\label{WGAN-GP penalty}
R(D) = \|1-\|\nabla D\| \|_{L^2(\Omega)}^2
\end{equation}
\item Other Lipschitz penalties \cite{kodali2017convergence,petzka2018regularization} in simplified forms
\begin{equation}
\label{WGAN-LP penalty}
R(D) = \|\max(0,\|\nabla D\|-1)\|_{L^2(\Omega)}^2 \quad \text{or} \quad \|\max(0,\|\nabla D\|^2-1)\|_{L^1(\Omega)}
\end{equation}
\end{itemize}
None of the above regularizations lead to a well-defined GAN loss.
\begin{proposition}[Ill-posedness]
\label{prop. ill-posedness}
Consider the empirical loss function
\begin{equation*}
L^{(n)}(P) = \sup_{D\in C^1(\Omega)} \E_{P_*^{(n)}}[D] -\E_P[D] - c R(D)
\end{equation*}
with any $c\geq 0$, where the regularization $R$ is any of (\ref{L2 penalty}, \ref{gradient penalty}, \ref{WGAN-GP penalty}, \ref{WGAN-LP penalty}).
Suppose the dimension $d \geq 3$.
Then, for any distribution $P\neq P_*^{(n)}$
\begin{equation*}
L^{(n)}(P) = \infty
\end{equation*}
\end{proposition}
Heuristically, these regularizations are too weak to control the complexity of $D$, so that $D$ can diverge around the sample points of $P_*^{(n)}$, and thus the loss blows up.
By the universal approximation property \cite{hornik1991approximation}, this result holds if we implement $D$ by neural networks or random feature functions.

It follows that the two-time-scale training, with $D$ trained to optimality, is not applicable.

\subsection{Generalization error}
Nevertheless, we show that one-time-scale training still performs well and achieves a small generalization error.

For simplicity, we focus on the $L^2$ regularization (\ref{L2 penalty}),
\begin{equation}
\label{WGAN+L2}
\min_p \max_D L(p, D) = \min_p \max_D \E_{P_*}[D] -\E_P[D] - \frac{c}{2}\|D\|_{L^2(\Omega)}^2
\end{equation}
with some $c > 0$.
As usual we model $D$ by the random feature function (\ref{RFM discriminator}).

The one-time-scale gradient descent-ascent (\ref{p gradient descent}, \ref{D gradient ascent}) can be written as
\begin{align}
\label{one time scale GD}
\frac{d}{dt} p_t &= D_t,
\quad \frac{d}{dt} D_t = k*(P_*-P_t) - ck*D_t
\end{align}
Again, denote by $p_t^{(n)}, D_t^{(n)}$ the training trajectory on the empirical loss $L^{(n)}$.

\begin{theorem}[Generalization error]
\label{thm. one-time-scale generalization error}
Suppose $0 < c \leq \sqrt{2}$.
Initialize $p_t^{(n)}$ by $p_0$ and the parameter function of the discriminator $D_t^{(n)}$ by $a_0^{(n)} \equiv 0$.
With probability $1-\delta$ over the sampling of $P_*^{(n)}$, we have
\begin{equation*}
W_2\big(\Pi_{\Delta}(p_t^{(n)}),P_*\big) \leq \sqrt{\frac{d}{c}} \frac{\|p_*-p_0\|_{\HS}}{\sqrt{t}} + {\color{black}\sqrt{\frac{d}{c}}~ \frac{4\sqrt{2\log 2d} + \sqrt{2\log (2/\delta)}}{\sqrt{n}} t^{3/2}}
\end{equation*}
\end{theorem}
The condition $c \leq \sqrt{2}$ is imposed for convenience.
Otherwise, the ``friction" is too large, so the convergence becomes slower and the formula becomes more complicated.

Similar to Corollary \ref{cor. early stopping}, we show that early stopping can escape from the curse of dimensionality.
\begin{corollary}[Early stopping]
\label{cor. one-time-scale early stopping}
Given the condition of Theorem \ref{thm. one-time-scale generalization error}, if we choose an early-stopping time $T$ as follows
\begin{equation*}
T \asymp \|p_*-p_0\|_{\mathcal{H}}^{1/2} \big(\frac{n}{\log d}\big)^{1/4}
\end{equation*}
then the generalization error obeys
\begin{align*}
W_2\big(P_*, \Pi_{\Delta}(p_T^{(n)})\big) &\lesssim {\color{black}\sqrt{d} \|p_*-p_0\|_{\mathcal{H}}^{3/4} \Big(\frac{\log d}{n}\Big)^{1/8}}
\end{align*}
\end{corollary}

Finally, we can also establish memorization in the long time limit.
\begin{proposition}[Memorization]
\label{prop. memorization one-time-scale}
Given the condition of Theorem \ref{thm. one-time-scale generalization error}, $P_t^{(n)}$ converges weakly to $P_*^{(n)}$.
\end{proposition}

\section{Slow Deterioration}
\label{sec. slow deterioration}

In the previous sections, we demonstrated that the generalization ability of the adversarial density estimation models, (\ref{WGAN + RKHS}) and (\ref{WGAN+L2}), can be attributed to the dimension-independent complexity of the discriminators during training (e.g. Rademacher complexity).
This section provides a supplementary view, that because the complexity of $D_t$ grows slowly, it would take a very long time for $D_t$ to deteriorate to the optimal Lipschitz discriminator of WGAN \cite{arjovsky2017wasserstein}.
Heuristically, this slow deterioration is beneficial, because as illustrated in Figure \ref{fig: W2 geodesic} curve \textcircled{1}, training on the Wasserstein landscape suffers from the curse of dimensionality.
\footnote{Technically, Figure \ref{fig: W2 geodesic} curve \textcircled{1} concerns the $W_2$ loss, but it is reasonable to believe that training on $W_1$ or any $W_p$ loss cannot escape from the curse of dimensionality either.}

Consider the empirical WGAN loss with Lipschitz discriminator
\begin{equation*}
\max_{\|D\|_{\Lip}\leq 1} \E_{P_*^{(n)}}[D] -\E_P[D]
\end{equation*}
The Kantorovich-Rubinstein theorem \cite{villani2003topics} tells us that this is the $W_1$ metric between $P$ and $P_*^{(n)}$, while the Arzel\`{a}–Ascoli theorem implies that the maximizers $D_*$ exist (over the compact domain $\Omega$).
It has been the focus of several GAN models (e.g. \cite{arjovsky2017wasserstein,gulrajani2017improved,kodali2017convergence,petzka2018regularization}) to try to learn these $D_*$ by neural networks.

For convenience, suppose our modeled distribution $P$ is exactly $P_*$ and $P_*$ is the uniform distribution over $\Omega=[0,1]^d$.
By Theorem 5.1 of \cite{dobric1995asymptotics}, we have
\begin{equation}
\label{W1 empirical}
\max_{\|D\|_{\Lip}\leq 1} \E_{P_*^{(n)}}[D] -\E_P[D] = W_1(P_*,P_*^{(n)}) \asymp n^{-1/d}
\end{equation}
The large gap $n^{-1/d}$ indicates that $D_*$ well separates $P_*$ and $P_*^{(n)}$, suggesting that during GAN training, $D_*$ can quickly drive $P$ away from $P_*$ and towards memorizing $P_*^{(n)}$.

Let us consider the loss associated with the discriminator
\begin{align*}
\max_a L^{(n)}(a) &= \E_{P_*^{(n)}}[D] -\E_P[D] - R(a)\\
D(\x) &= \E_{\rho_0}[a(\w,b)~ \sigma(\w\cdot\x+b)]
\end{align*}
where $D$ is again modeled as a random feature function.
Assume two-time-scale training, so that $P=P_*$ is fixed as we train $D_t$.
The regularization $R$ can be, for instance, the Lipschitz constraint
\begin{equation}
\label{Lipschitz constraint}
R(a) = \begin{cases}
0 \text{ if } \|D\|_{\Lip} \leq 1\\
\infty \text{ else}
\end{cases}
\end{equation}
or a Lipschitz penalty, as in WGAN-GP \cite{gulrajani2017improved, kodali2017convergence, petzka2018regularization}
\begin{equation}
\label{Lipschitz penalty}
R(a) = c \max(0,\|D\|_{\Lip}-1)
\end{equation}
with $c \gg 1$.

The following result indicates that $D_t$ cannot approximate $D_*$ efficiently.
\begin{proposition}[Slow deterioration]
\label{prop. slow deterioration}
Suppose $R$ is any regularization term such that $L^{(n)}$ is bounded above and that we can train $a$ by continuous-time (sub)gradient flow.
Let $a_t$ be the gradient ascent trajectory with any initialization $a_0 \in L^2(\rho_0)$, let $D_t$ be the discriminator, and let $D_*$ be any maximizer of (\ref{W1 empirical}).
With probability $1-\delta$ over the sampling of $P_*^{(n)}$, we have
\begin{equation*}
\|D_t-D_*\|_{L^{\infty}(\Omega)} \geq \frac{3}{40} n^{-1/d} - \frac{2\sqrt{2\log 2d} + \sqrt{\log(2/\delta)/2}}{\sqrt{n}}o(\sqrt{t})
\end{equation*}
\end{proposition}

Both (\ref{Lipschitz constraint}) and (\ref{Lipschitz penalty}) satisfy the condition of Proposition \ref{prop. slow deterioration}, since $-L^{(n)}$ becomes a proper closed convex function.

Hence, it takes at least $\omega(n^{1-\frac{2}{d}})$ time to learn the optimal Lipschitz discriminator.
The generality of Proposition \ref{prop. slow deterioration} indicates that it could be futile to search for a Lipschitz regularization for WGAN.


\section{Proofs}
\label{sec. proofs}

The results from Sections \ref{sec. two time scale}, \ref{sec. one time scale} and \ref{sec. slow deterioration} are proved in the following three subsections respectively.

\subsection{Two time scale training}

\begin{definition}
\label{def. RKHS eigendecomposition}
Let $k$ be the kernel defined in (\ref{RKHS kernel}).
Consider the convolution (\ref{k convolution}) as a symmetric compact operator over $L^2(\Omega)$.
By universal approximation (Section \ref{sec. function representation}), $k$ is positive definite.
It follows that we can construct an orthonormal basis of eigenvectors $\{\e_i\}_{i=1}^{\infty}$ with eigenvalues $\lambda_i > 0$.
Denote
\begin{equation*}
\tilde{\e}_i = \frac{\e_i}{\sqrt{\lambda_i}}
\end{equation*}
Then, $\{\tilde{\e}_i\}_{i=1}^{\infty}$ is an orthonormal basis of $\HS$.
\end{definition}

Regarding the Wasserstein test loss (\ref{W2 loss}), we have the following convenient bound.

\begin{lemma}
\label{lemma. W2-L2 bound}
For any $p,q \in L^2(\Omega)$,
\begin{equation*}
W_2\big(\Pi_{\Delta}(p), \Pi_{\Delta}(q)\big) \leq \sqrt{d} \|p-q\|_{L^2(\Omega)}
\end{equation*}
\end{lemma}

\begin{proof}
The $L^1$ difference between two probability densities is equivalent to an optimal transport distance with the loss
\begin{equation*}
c(\x,\y) = \begin{cases}0 ~\text{if}~ \x=\y\\ 1 ~\text{else}\end{cases}
\end{equation*}
{\color{black}
Since $\|\x-\y\| \leq \text{diam}(\Omega) c(\x,\y)$ for all $\x,\y \in \Omega$,
the $W_2$ metric is dominated by the $L^1$ distance}
\begin{align*}
W_2\big(\Pi_{\Delta}(p), \Pi_{\Delta}(q)\big) &\leq \sqrt{d} \|\Pi_{\Delta}(p)-\Pi_{\Delta}(q)\|_{L^1(\Omega)}
\end{align*}
Meanwhile, since $\Omega=[0,1]^d$ has unit volume
\begin{align*}
\|\Pi_{\Delta}(p)-\Pi_{\Delta}(q)\|_{L^1(\Omega)}
&\leq \|\Pi_{\Delta}(p)-\Pi_{\Delta}(q)\|_{L^2(\Omega)} \leq \|p-q\|_{L^2(\Omega)}
\end{align*}
\end{proof}

\noindent
We also need the following lemma from \cite{yang2020generalization}.
\begin{lemma}
\label{lemma. RKHS Monte Carlo rate}
For any distribution $P_* \in \PS(\Omega)$ and any $\delta \in (0,1)$, with probability $1-\delta$ over the i.i.d. sampling of $P_*^{(n)}$,
\begin{align*}
\sup_{\|\w\|_1+|b| \leq 1} \Big|\int \sigma(\w\cdot\x+b) ~d(P_*-P_*^{(n)})(\x)\Big| \leq \frac{4\sqrt{2\log 2d} + \sqrt{2\log (2/\delta)}}{\sqrt{n}}
\end{align*}
\end{lemma}

\subsubsection{Proof of generalization gap}
\begin{proof}[Proof of Proposition \ref{prop. generalization gap}]
First, for the plain gradient flow (\ref{MMD gradient descent}), we have
\begin{align*}
\frac{d}{dt} \|p^{(n)}_t-p_t\|_{L^2(\Omega)} &= \big\lb \frac{p^{(n)}_t-p_t}{\|p^{(n)}_t-p_t\|}, ~k*(p_*^{(n)}-p_t^{(n)}) - k*(p_*-p_t) \big\rb_{L^2(\Omega)}\\
&= \big\lb \frac{p^{(n)}_t-p_t}{\|p^{(n)}_t-p_t\|}, ~k*(p_*^{(n)}-p_*) - k*(p^{(n)}_t-p_t) \big\rb_{L^2(\Omega)}\\
&\leq \big\lb \frac{p^{(n)}_t-p_t}{\|p^{(n)}_t-p_t\|}, ~k*(p_*^{(n)}-p_*) \big\rb\\
&\leq \|k*(p_*^{(n)}-p_*)\|_{L^2(\Omega)}\\
&\leq \sup_{\x\in\Omega} \Big| \E_{\rho_0(\w,b)}\Big[ \sigma(\w\cdot\x+b) \int \sigma(\w\cdot\x'+b) ~d(P_*^{(n)}-P_*)(\x') \Big]\Big|\\
&\leq \sup_{\|\w\|_1+|b| \leq 1} \Big|\int \sigma(\w\cdot\x+b) ~d(P_*-P_*^{(n)})(\x)\Big|
\end{align*}
Then, Lemma \ref{lemma. RKHS Monte Carlo rate} implies that with probability $1-\delta$,
\begin{equation}
\label{L2 generalization gap}
\|p^{(n)}_t-p_t\|_{L^2(\Omega)} \leq \int_0^t \frac{4\sqrt{2\log 2d} + \sqrt{2\log (2/\delta)}}{\sqrt{n}}
\end{equation}
One can conclude by Lemma \ref{lemma. W2-L2 bound} that
\begin{equation*}
W_2\big(\Pi_{\Delta}(p_t^{(n)}), \Pi_{\Delta}(p_t)\big) \leq \sqrt{d}~ \frac{4\sqrt{2\log 2d} + \sqrt{2\log (2/\delta)}}{\sqrt{n}}t
\end{equation*}

Next, for the projected gradient flow (\ref{MMD projected gradient descent}), we have
\begin{align*}
\frac{d}{dt} \|p^{(n)}_t-p_t\| &= \Big\lb \frac{p^{(n)}_t-p_t}{\|p^{(n)}_t-p_t\|}, ~\Pi_{T_{p_t^{(n)}}\Delta}\big(k*(p_*^{(n)}-p_t^{(n)})\big) - \Pi_{T_{p_t}\Delta}\big(k*(p_*-p_t)\big) \Big\rb\\
&= \lim_{\epsilon\to 0^+} \Big\lb \frac{p^{(n)}_t-p_t}{\|p^{(n)}_t-p_t\|}, ~\frac{\Pi_{T\Delta}\big(p_{t,\epsilon}^{(n)}\big)-p_t^{(n)}}{\epsilon} - \frac{\Pi_{T\Delta}\big(p_{t,\epsilon}\big)-p_t}{\epsilon} \Big\rb\\
\end{align*}
where
\begin{align*}
p_{t,\epsilon} := p_t + \epsilon k*(p_*-p_t),\quad p_{t,\epsilon}^{(n)} := p_t^{(n)} + \epsilon k*(p_*^{(n)}-p_t^{(n)})
\end{align*}
Since $\Pi_{\Delta}$ is a projection onto a convex set,
\begin{align*}
\big\lb p_t^{(n)}-p_t, ~p_{t,\epsilon}-\Pi_{\Delta}(p_{t,\epsilon}) \big\rb
&\leq \big\lb \Pi_{\Delta}(p_{t,\epsilon})-p_t, ~p_{t,\epsilon}-\Pi_{\Delta}(p_{t,\epsilon}) \big\rb = O(\epsilon^2)\\
\big\lb p_t-p^{(n)}_t, ~p^{(n)}_{t,\epsilon}-\Pi_{\Delta}(p^{(n)}_{t,\epsilon}) \big\rb
&= O(\epsilon^2)
\end{align*}
It follows that
\begin{align*}
\frac{d}{dt} \|p^{(n)}_t-p_t\| &\leq \lim_{\epsilon\to 0^+} \Big\lb \frac{p^{(n)}_t-p_t}{\|p^{(n)}_t-p_t\|}, ~\frac{p_{t,\epsilon}^{(n)}-p_t^{(n)}}{\epsilon} - \frac{p_{t,\epsilon}-p_t}{\epsilon} \Big\rb + O(\epsilon)\\
&= \Big\lb \frac{p^{(n)}_t-p_t}{\|p^{(n)}_t-p_t\|}, ~k*(p_*-p_t)-k*(p^{(n)}_*-p^{(n)}_t) \Big\rb
\end{align*}
The proof is completed using the same argument for plain gradient flow.
\end{proof}

\subsubsection{Proof of generalization error}
\begin{lemma}[Training error, two-time-scale]
\label{lemma. training error}
If $p_t$ is trained by the gradient flow (\ref{MMD gradient descent}) with any target distribution $p_* \in L^2(\Omega)$, we have
{\color{black}
\begin{align*}
\|p_t-p_*\|_{L^2(\Omega)}^2 \leq \frac{\|p_0-p_*\|_{\HS}^2}{t}
\end{align*}
}
\end{lemma}

\begin{proof}
We show that the training rule (\ref{MMD gradient descent}) coincides with the training trajectory of RKHS regression:
For any $a \in L^2(\rho_0)$, define the function
\begin{equation*}
f_a(\x) = \E_{\rho_0(\w,b)}[a(\w,b)\sigma(\w\cdot\x+b)]
\end{equation*}
If $\|p_0-p_*\|_{\HS} < \infty$, we can choose $a_0$ such that $f(a_0) = p_0-p_*$.
Then, if we train $a_t$ by gradient flow with initialization $a_0$ on the regression loss
\begin{equation*}
\min_a \Gamma(a) = \frac{1}{2} \|f_a\|_{L^2(\Omega)}^2,
\end{equation*}
the function $f_{a_t}$ evolves by
\begin{align*}
\frac{d}{dt} f_{a_t}(\x) &= \E_{\rho_0}\big[\frac{d}{dt}a_t~ \sigma(\w\cdot\x+b)\big]\\
&= \E_{\rho_0}\big[ -\int_{\Omega} f_{a_t}(\x') \sigma(\w\cdot\x'+b) d\x' \sigma(\w\cdot\x+b)\big]\\
&= -\int_{\Omega} f_{a_t}(\x') k(\x, \x') d\x'
\end{align*}
or equivalently
\begin{equation*}
\frac{d}{dt}f_{a_t} = -k * f_{a_t}
\end{equation*}
So the training dynamics of $f_{a_t}$ is the same as the training rule (\ref{MMD gradient descent}) for the function $p_t-p_*$.
Since $f_{a_0} = p_0-p_*$, we have $f_{a_t} = p_t-p_*$ for all $t\geq 0$.

It follows from the convexity of $\Gamma$ that
\begin{equation*}
\|p_t-p_*\|^2_{L^2(\Omega)} = \|f(a_t)\|^2_{L^2(\Omega)} \leq \frac{\|a_0\|^2_{\rho_0}}{t} = \frac{\|p_0-p_*\|_{\HS}^2}{t}
\end{equation*}
\end{proof}

\begin{proof}[Proof of Theorem \ref{thm. generalization error}]
Decompose the generalization error into training error + generalization gap:
\begin{align*}
\|p_*-p_t^{(n)}\|_{L^2(\Omega)} \leq \|p_*-p_t\|_{L^2} + \|p_t-p_t^{(n)}\|_{L^2}
\end{align*}
The first term is bounded by Lemma \ref{lemma. training error} and the second term is bounded by (\ref{L2 generalization gap}).
Therefore,
\begin{align*}
\|p_*-p_t^{(n)}\|_{L^2(\Omega)} \leq \frac{\|p_0-p_*\|_{\HS}}{\sqrt{t}} + \frac{4\sqrt{2\log 2d} + \sqrt{2\log (2/\delta)}}{\sqrt{n}}t
\end{align*}
Then, we conclude by Lemma \ref{lemma. W2-L2 bound}.
\end{proof}

\subsubsection{Proof of memorization}

Proposition \ref{prop. memorization} is a corollary of the following lemma.

\begin{lemma}[Universal convergence, two-time-scale]
\label{lemma. universal convergence}
Given any signed measure $\tilde{P} \in \M(\Omega)$ and any initialization $p_0 \in L^2(\Omega)$, if we define $P_t$ by
\begin{equation*}
\frac{d}{dt} p_t = k * (\tilde{P}-P_t)
\end{equation*}
with any initialization $p_0 \in L^2(\Omega)$, then $P_t$ converges weakly to $\tilde{P}$.
\end{lemma}

\begin{proof}
Let $\lambda_i$ and $\tilde{\e}_i$ be the eigendecomposition in Definition \ref{def. RKHS eigendecomposition}.
Denote $f_t = k*P_t$ and $\tilde{f} = k*\tilde{P}$.
Decompose $f_t-\tilde{f}$ into
\begin{equation*}
f_t-\tilde{f} = \sum_{i=1}^{\infty} y_t^i \tilde{\e}_i
\end{equation*}
Then,
\begin{equation*}
\frac{d}{dt} (f_t - \tilde{f}) = k*[k*(\tilde{P}-P_t)] = -k*(f_t - \tilde{f}) = -\sum_{i=1}^{\infty} y_t^i \lambda_i \tilde{\e}_i
\end{equation*}
It follows that $y_t^i = y_0^i e^{-\lambda_i t}$.
Given that
\begin{equation*}
\sum_{i=1}^{\infty} (y_0^i)^2 = \|f_0 - \tilde{f}\|_{\HS}^2 = \E_{(P_0-\tilde{P})^2}[k] \leq (\|P_0\|_{\TV}+\|\tilde{P}\|_{\TV})^2 \|k\|_{C(\Omega\times\Omega)} < \infty
\end{equation*}
we can apply dominated convergence theorem to obtain
\begin{equation*}
\lim_{t\to\infty} \|f_t - \tilde{f}\|_{\HS}^2 = \lim_{t\to\infty} \sum_{i=1}^{\infty} (y_0^i)^2 e^{-2\lambda_i t} = 0
\end{equation*}
Thus, $f_t \to \tilde{f}$ in $\HS$, which implies
\begin{equation*}
\forall f \in \HS, \quad \lim_{t\to\infty} \int f ~d(P_t-\tilde{P}) = \lim_{t\to\infty} \lb f_t-\tilde{f}, ~f\rb_{\HS} = 0
\end{equation*}
As discussed in Section \ref{sec. function representation}, the RKHS space $\HS$ is dense in $C(\Omega)$ under the supremum norm.
It follows that
\begin{equation*}
\forall f \in C(\Omega), \quad \lim_{t\to\infty} \int f ~d P_t = \int f ~d\tilde{P}
\end{equation*}
Hence, $P_t$ converges weakly to $\tilde{P}$.
\end{proof}

\subsection{One time scale training}

\subsubsection{Proof of ill-posedness}

\begin{proof}[Proof of Proposition \ref{prop. ill-posedness}]
Let $\{\x_i\}_{i=1}^n$ be the sample points of $P_*^{(n)}$.
Define the function
\begin{equation*}
D(\x) = \sum_{i=1}^n \|\x-\x_i\|^{-d/2 + 1.1}
\end{equation*}
Let $\eta$ be a mollifier ($\eta$ is smooth, non-negative, supported in the unit ball and $\int \eta =1$).
Define
\begin{equation*}
D_{\epsilon} = \int \eta(\y) D(\x-\epsilon\y) d\y
\end{equation*}
Then, $\sup_{\epsilon > 0} R(D_{\epsilon}) < \infty$, while for any $P \neq P_*^{(n)}$,
\begin{equation*}
\lim_{\epsilon\to 0^+} E_{P_*^{(n)}}[D_{\epsilon}] - \E_P[D_{\epsilon}] = \infty
\end{equation*}
It follows that
\begin{equation*}
\sup_{D \in C^1} \E_{P_*^{(n)}}[D] - \E_P[D] - R(D) = \infty
\end{equation*}
\end{proof}

\subsubsection{Proof of generalization error}

As usual, we try to bound the generalization error by the training error plus the generalization gap, and estimate them separately.

\begin{lemma}[Duhamel's integral]
\label{lemma. second order ODE}
Consider the one-dimensional second-order ODE
\begin{align*}
\ddot{x} + b\dot{x} + ax &= q\\
x(0) = x_0, \quad \dot{x}(0) &= 0
\end{align*}
where $a,b$ are constants and $q$ is a locally integrable function in $t$.
If $4a>b^2$, then the solution is given by
\begin{equation*}
x(t) = x_0 e^{-\frac{b}{2}t}\Big[\cos\big(t\sqrt{a-\frac{b^2}{4}}\big) + \frac{\frac{b}{2}\sin\big(t \sqrt{a-\frac{b^2}{4}}\big)}{\sqrt{a-\frac{b^2}{4}}}\Big] + \int_0^t q(s) e^{-\frac{b}{2}(t-s)} \frac{\sin\big((t-s)\sqrt{a-\frac{b^2}{4}}\big)}{\sqrt{a-\frac{b^2}{4}}} ds
\end{equation*}
\end{lemma}
\begin{proof}
Direct verification.
\end{proof}

\begin{lemma}
\label{lemma. k operator norm}
Let $k$ be the kernel defined in (\ref{RKHS kernel}).
Assume that the support of the parameter distribution $\rho_0$ is contained in $\{\|\w\|_1+|b|\leq 1\}$.
Then, the operator norm of the convolution (\ref{k convolution}) over $L^2(\Omega)$ is bounded by $1$.
\end{lemma}
\begin{proof}
\begin{align*}
\|k\|_{op} &= \sup_{\|f\|_{L^2(\Omega)} \leq 1} \lb k*f, f\rb_{L^2(\Omega)}\\
&= \sup_{\|f\|_{L^2} \leq 1} \E_{\rho_0(\w,b)}\Big[\Big(\int_{\Omega} f(\x) \sigma(\w\cdot\x+b)\Big)^2\Big]\\
&\leq \E_{\rho_0(\w,b)}\big[\|\sigma(\w\cdot\x+b)\|_{L^2}^2\big]\\
&\leq \sup_{\|\w\|_1+|b|\leq 1} \sup_{\x\in\Omega} \sigma(\w\cdot\x+b)^2\\
&\leq 1
\end{align*}
\end{proof}

\begin{lemma}[Training error, one-time-scale]
\label{lemma. one time scale training error}
Given any $0 < c \leq \sqrt{2}$,
any target density $p_*$ and any initialization $p_0$ such that $\|p_*-p_0\|_{\HS} < \infty$, let $p_t, D_t$ be the training trajectory (\ref{one time scale GD}) with $D_0 \equiv 0$, then we have
\begin{equation*}
\|p_*-p_t\|_{L^2} \leq \frac{\|p_*-p_0\|_{\HS}}{\sqrt{ct}}
\end{equation*}
\end{lemma}
\begin{proof}
Let $\{\lambda_i,\tilde{\e}_i\}_{i=1}^{\infty}$ be the eigendecomposition given by Definition \ref{def. RKHS eigendecomposition}.
Lemma \ref{lemma. k operator norm} implies that all $\lambda_i \in (0,1]$.
Then, the condition $0 < c \leq \sqrt{2}$ implies that $2\lambda_i \geq c^2\lambda_i^2$.

For any $t$, define the orthonormal decomposition
\begin{equation*}
p_t-p_* = \sum_{i=1}^{\infty} x^i(t) \tilde{\e}_i, \quad \sum_{i=1}^{\infty} x^i(0)^2 = \|p_0-p_*\|_{\HS}^2 < \infty
\end{equation*}
Denote $u(t) = p_t-p_*$.
The training rule (\ref{one time scale GD}) can be rewritten as
\begin{equation*}
\ddot{u} + ck*\dot{u} + k*u = 0
\end{equation*}
Taking RKHS inner product with $\tilde{\e}_i$ for each $i$, we obtain
\begin{align*}
\forall i, \quad \ddot{x}^i + c\lambda_i \dot{x}^i + \lambda_i x^i &= 0\\
x^i(0) = x^i_0, \quad \dot{x}^i(0) &= 0
\end{align*}
Since $4\lambda_i > c^2\lambda_i^2$, Lemma \ref{lemma. second order ODE} implies that
\begin{equation*}
\forall i, \quad x^i(t) = x^i_0 e^{-\frac{c\lambda_i}{2}t}\Big[\cos\big(t\sqrt{\lambda_i-\frac{c^2\lambda^2}{4}}\big) + \frac{\frac{c\lambda_i}{2}\sin\big(t \sqrt{\lambda_i-\frac{c^2\lambda_i^2}{4}}\big)}{\sqrt{\lambda_i-\frac{c^2\lambda_i^2}{4}}}\Big]
\end{equation*}
Since $\frac{4}{c^2\lambda_i} - 1 \geq 1$,
\begin{equation*}
|x^i(t)| \leq |x^i_0| e^{-\frac{c\lambda_i}{2}t}\Big|\cos\big(t\sqrt{\lambda_i-\frac{c^2\lambda^2}{4}}\big) + \sin\big(t \sqrt{\lambda_i-\frac{c^2\lambda_i^2}{4}}\big)\Big| \leq \sqrt{2} |x^i_0| e^{-\frac{c\lambda_i}{2}t}
\end{equation*}
It follows that
\begin{align*}
\|p_t-p_*\|_{L^2}^2 &= \sum_{i=1}^{\infty} \lambda_i x^i(t)^2 \leq \sum_{i=1}^{\infty} 2\lambda_i (x^i_0)^2 e^{-c\lambda_i t}\\
&\leq \sum_{i=1}^{\infty} \sup_{\lambda>0} 2\lambda (x^i_0)^2 e^{-c\lambda t} \leq \sum_{i=1}^{\infty} \frac{2}{e} \frac{1}{ct} (x^i_0)^2\\
&\leq \frac{\|p_*-p_0\|^2_{\HS}}{ct}
\end{align*}
\end{proof}

Recall that the one-time-scale training over the empirical loss $L^{(n)}$ is given by the following dynamics
\begin{align}
\label{one time scale GD empirical}
\frac{d}{dt} p^{(n)}_t &= D^{(n)}_t,
\quad \frac{d}{dt} D^{(n)}_t = k*(P_*^{(n)}-P^{(n)}_t) - ck*D^{(n)}_t
\end{align}

\begin{lemma}[Generalization gap, one-time-scale]
\label{lemma. one time scale generalization gap}
Given any $0 < c < 2$ and any target distribution $P_*$,
let $p_t$ and $p_t^{(n)}$ be the trajectory of the dynamics (\ref{one time scale GD}) and (\ref{one time scale GD empirical}) with the same initialization $p_0 = p_0^{(n)}$ and $D_0 = D_0^{(n)} \equiv 0$.
Then, with probability $1-\delta$ over the sampling of $P_*^{(n)}$,
we have
\begin{equation*}
\|p_t-p_t^{(n)}\|_{L^2(\Omega)} \leq {\color{black}\frac{4\sqrt{2\log 2d} + \sqrt{2\log (2/\delta)}}{\sqrt{n}} \frac{t^{3/2}}{\sqrt{c}}}
\end{equation*}
\end{lemma}
\begin{proof}
Let $\{\lambda_i,\tilde{\e}_i\}_{i=1}^{\infty}$ be the eigendecomposition given by Definition \ref{def. RKHS eigendecomposition}.
Lemma \ref{lemma. k operator norm} implies that $4\lambda_i > c^2\lambda_i^2$.

For any $t$, define the orthonormal decompositions
\begin{align*}
p_t-p_t^{(n)} &= \sum_{i=1}^{\infty} y^i(t) \tilde{\e}_i\\
k*(P_*-P_*^{(n)}) &= \sum_{i=1}^{\infty} q^i \tilde{\e}_i
\end{align*}
Denote $u(t) = p_t-p_t^{(n)}$.
Then, $u(0) = \dot{u}(0) \equiv 0$.
The training rules (\ref{one time scale GD}, \ref{one time scale GD empirical}) imply that
\begin{equation*}
\ddot{u} + ck*\dot{u} + k*u = k*(P_*-P_*^{(n)})
\end{equation*}
Taking RKHS inner product with $\tilde{\e}_i$ for each $i$, we obtain
\begin{align*}
\forall i, \quad \ddot{y}^i + c\lambda_i \dot{y}^i + \lambda_i y^i &= q^i\\
y^i(0) = \dot{y}^i(0) &= 0
\end{align*}
Since $4\lambda_i > c^2\lambda_i^2$, Lemma \ref{lemma. second order ODE} implies that
\begin{equation*}
\forall i, \quad y^i(t) = q^i \int_0^t e^{-\frac{c\lambda_i}{2}(t-s)} \frac{\sin\big((t-s) \sqrt{\lambda_i-\frac{c^2\lambda_i^2}{4}}\big)}{\sqrt{\lambda_i-\frac{c^2\lambda_i^2}{4}}} ds
\end{equation*}
{\color{black}Then
\begin{align*}
\sqrt{\lambda_i} |y^i(t)| &\leq |q^i| \sqrt{\lambda_i} \int_0^t e^{-\frac{c\lambda_i}{2}(t-s)}(t-s) ds\\
&\leq |q^i| \sqrt{\lambda_i} \Big(\frac{c\lambda_i}{2}\Big)^{-2} \big(1-e^{-\frac{c\lambda_i}{2}t} - e^{-\frac{c\lambda_i}{2}t} \frac{c\lambda_i}{2}t \big)\\
&\leq |q^i| \sqrt{\frac{2}{c}} t^{3/2} z_i^{-3/2} [1-e^{-z_i}(1+z_i)], \quad z_i :=  \frac{c\lambda_i}{2}t\\
&\leq |q^i| \sqrt{\frac{2}{c}} t^{3/2} \sup_{z>0} z^{-3/2} [1-e^{-z}(1+z)]\\
&\leq |q^i| \frac{t^{3/2}}{\sqrt{c}}
\end{align*}
It follows that
\begin{align*}
\|p_t-p_t^{(n)}\|_{L^2(\Omega)}^2 &= \sum_{i=1}^{\infty} \lambda_i y^i(t)^2 \leq \sum_{i=1}^{\infty} (q^i)^2 \frac{t^3}{c} = \frac{t^3}{c} \|k*(P_*-P_*^{(n)})\|_{\HS}^2
\end{align*}
Lemma \ref{lemma. RKHS Monte Carlo rate} implies that with probability $1-\delta$,
\begin{align*}
\|k*(P_*-P_*^{(n)})\|_{\HS}^2 &= \iint k ~d(P_*-P_*^{(n)})^2\\
&\leq \E_{\rho_0(\w,b)} \Big[\Big|\int \sigma(\w\cdot\x+b) ~d(P_*-P_*^{(n)})(\x) \Big|^2\Big]\\
&\leq \Big(\frac{4\sqrt{2\log 2d} + \sqrt{2\log (2/\delta)}}{\sqrt{n}}\Big)^2
\end{align*}
Hence,
\begin{align*}
\|p_t-p_t^{(n)}\|_{L^2(\Omega)}^2 &\leq \Big(\frac{4\sqrt{2\log 2d} + \sqrt{2\log (2/\delta)}}{\sqrt{n}}\Big)^2 \frac{t^3}{c}
\end{align*}
}
\end{proof}

\begin{proof}[Proof of Theorem \ref{thm. one-time-scale generalization error}]
As usual, we bound the generalization error by training error plus generalization gap:
\begin{align*}
\|p_*-p_t^{(n)}\|_{L^2(\Omega)} \leq \|p_*-p_t\|_{L^2} + \|p_t-p_t^{(n)}\|_{L^2}
\end{align*}
where we assume the same initialization $p_0 = p_0^{(n)}$ and $a_0 = a_0^{(n)} \equiv 0$.
The first term is bounded by Lemma \ref{lemma. one time scale training error} and the second term is bounded by Lemma \ref{lemma. one time scale generalization gap}.
Therefore,
\begin{align*}
\|p_*-p_t^{(n)}\|_{L^2(\Omega)} \leq \frac{\|p_*-p_0\|_{\HS}}{\sqrt{ct}} + \frac{4\sqrt{2\log 2d} + \sqrt{2\log (2/\delta)}}{\sqrt{n}} \frac{t^{3/2}}{\sqrt{c}}
\end{align*}
The proof is completed using Lemma \ref{lemma. W2-L2 bound}.
\end{proof}

\subsubsection{Proof of memorization}

Proposition \ref{prop. memorization one-time-scale} is a corollary of the following lemma.

\begin{lemma}[Universal convergence, one-time-scale]
\label{lemma. one-time-scale universal convergence}
Given any signed measure $\tilde{P} \in \M(\Omega)$, if we define $p_t$ by the dynamics
\begin{equation*}
\ddot{p}_t = - k * (p_t-\tilde{P}) - c k * \dot{p}_t
\end{equation*}
with any initialization $p_0 \in L^2(\Omega)$ and $\dot{p}_0 \equiv 0$, then $P_t$ converges weakly to $\tilde{P}$.
\end{lemma}

\begin{proof}
Let $\lambda_i$ and $\tilde{\e}_i$ be the eigendecomposition from Definition \ref{def. RKHS eigendecomposition}.
Define $u(t) = k*(P_t-\tilde{P})$ and
decompose $u(t)$ into
\begin{equation*}
u(t) = \sum_{i=1}^{\infty} y^i(t) \tilde{\e}_i
\end{equation*}
Then,
\begin{equation*}
\ddot{u} = k*[-k*(p_t-\tilde{P}) - ck*\dot{p}_t] = -k*u - ck*\dot{u} = \sum_{i=1}^{\infty} -\lambda_i (y^i+c\dot{y}^i) \tilde{\e}_i
\end{equation*}
It follows that
\begin{equation*}
\forall i, \quad \ddot{y}^i + c\lambda_i \dot{y}^i + \lambda_i y^i = 0
\end{equation*}
Using the argument from the proof of Lemma \ref{lemma. one time scale training error}, we obtain
\begin{equation*}
|y^i(t)| \leq \sqrt{2} |y^i(0)| e^{-\frac{c\lambda_i}{2}t}
\end{equation*}
Since
\begin{equation*}
\sum_{i=1}^{\infty} y_i(0)^2 = \|u_0\|_{\HS}^2 = \E_{(P_0-\tilde{P})^2}[k] \leq \|P_0-\tilde{P}\|_{\TV}^2 \|k\|_{C(\Omega\times\Omega)} < \infty
\end{equation*}
we can apply dominated convergence theorem to obtain
\begin{equation*}
\lim_{t\to\infty} \|u_t\|_{\HS}^2 = \lim_{t\to\infty} 2 \sum_{i=1}^{\infty} y^i(0)^2 e^{-c\lambda_i t} = 0
\end{equation*}
Thus, $u_t \to 0$ in $\HS$, which implies
\begin{equation*}
\forall f \in \HS, \quad \lim_{t\to\infty} \int f ~d(P_t-\tilde{P}) = \lim_{t\to\infty} \lb u_t, ~f\rb_{\HS} = 0
\end{equation*}
Since $\HS$ is dense in $C(\Omega)$ in the supremum norm, $P_t$ converges weakly to $\tilde{P}$.
\end{proof}

\subsection{Slow deterioration}

\begin{lemma}
\label{lemma. Lipschitz optimizer}
With any set of $n$ points $\x_i \in \Omega$,
\begin{equation*}
\sup_{\|D\|_{\Lip} \leq 1} \big|\E_{P_*}[D] - \frac{1}{n}\sum_{i=1}^n D(\x_i)\big| \geq \frac{3}{20} n^{-1/d}
\end{equation*}
\end{lemma}
\begin{proof}
By Lemma 3.1 of \cite{e2020kolmogorov}, we have
\begin{equation}
\label{Lipschitz optimizer lower bound}
\sup_{\|D\|_{\Lip} \leq 1} \big|\E_{P_*}[D] - \frac{1}{n}\sum_{i=1}^n D(\x_i)\big| \geq \frac{d}{d+1} \frac{1}{[(d+1)\omega_d]^{1/d}}\frac{1}{\text{diam}(\Omega)} n^{-1/d}
\end{equation}
where $\omega_d$ is the volume of the unit ball in $\R^d$,
\begin{equation*}
\omega_d = \frac{\pi^{\frac{d}{2}}}{\Gamma(\frac{d}{2}+1)} \asymp \frac{1}{\sqrt{\pi d}} \Big(\frac{2\pi e}{d}\Big)^{d/2}
\end{equation*}
Denote the right hand side of (\ref{Lipschitz optimizer lower bound}) by $C_d n^{-1/d}$.
By direct computation, $C_d \geq 3/20$.
\end{proof}

\begin{lemma}
\label{lemma. norm growth rate}
Given the condition of Proposition \ref{prop. slow deterioration}, we have
\begin{equation*}
\lim_{t\to\infty} \frac{\|a_t\|_{L^2(\rho_0)}}{\sqrt{t}} = 0
\end{equation*}
\end{lemma}
\begin{proof}
The arguments are adapted from Lemma 3.3 of \cite{wojtowytsch2020convergence}.
Define the value
\begin{equation*}
L_{\infty}^{(n)} = \lim_{t\to\infty} L^{(n)}(a_t) \leq \sup_a L^{(n)} < \infty
\end{equation*}
which is assumed bounded by Proposition \ref{prop. slow deterioration}.
Meanwhile, we also have
\begin{align*}
\frac{d}{dt}\|a_t\|_{L^2(\rho_0)} &= \big\lb \frac{a_t}{\|a_t\|}, ~\delta_a L^{(n)}(a_t)\big\rb \leq \|\delta_a L^{(n)}(a_t)\|_{L^2} = \sqrt{\frac{d}{dt} L^{(n)}(a_t)}
\end{align*}
Thus, for any $t > t_0 \geq 0$, we have
\begin{align*}
\|a_t\|-\|a_{t_0}\| &\leq \int_{t_0}^t \sqrt{\frac{d}{dt} L^{(n)}(a_s)} ds \leq \sqrt{t-t_0} \sqrt{L^{(n)}(a_t)-L^{(n)}(a_{t_0})}\\
&\leq \sqrt{t-t_0} \sqrt{L^{(n)}_{\infty}-L^{(n)}(a_{t_0})}
\end{align*}
By choosing $t_0$ sufficiently large, the term $\sqrt{L^{(n)}_{\infty}-L^{(n)}(a_{t_0})}$ can be made arbitrarily small.
\end{proof}

\begin{proof}[Proof of Proposition \ref{prop. slow deterioration}]
Lemma \ref{lemma. RKHS Monte Carlo rate} implies that with probability $1-\delta$,
\begin{align*}
\forall r > 0, &\quad \sup_{\|D\|_{\HS} \leq r} \big|\E_{P_*}[D]-\E_{P_*^{(n)}}[D]\big|\\
&= \sup_{\|a\|_{L^2(\rho_0)} \leq r} \Big| \int \big[\E_{\rho_0}[a(\w,b)\sigma(\w\cdot\x+b)] ~d(P_*-P_*^{(n)})(\x) \Big|\\
&\leq \sup_{\|a\|_{L_2(\rho_0)} \leq r} \|a\|_{L^1(\rho_0)} \Big\|[ \int \sigma(\w\cdot\x+b) ~d(P_*-P_*^{(n)})(\x) \Big\|_{L^{\infty}(\rho_0)}\\
&\leq r \Big(4\sqrt{\frac{2\log 2d}{n}} + \sqrt{\frac{2\log (2/\delta)}{n}}\Big)
\end{align*}
Meanwhile, Lemma \ref{lemma. Lipschitz optimizer} implies that
\begin{align*}
\|D_t-D_*\|_{L^{\infty}(\Omega)} &\geq \frac{1}{2}\Big|\int D_t-D_* ~d(P_*-P_*^{(n)}) \Big|\\
&\geq \frac{1}{2}\Big(\frac{3}{20}n^{-1/d} - \Big|\int D_t ~d(P_*-P_*^{(n)})\Big|\Big)
\end{align*}
Combining the two inequalities, we obtain with probability $1-\delta$,
\begin{equation}
\label{Linf gap via norm}
\|D_t-D_*\|_{L^{\infty}(\Omega)} \geq \frac{1}{2}\Big[\frac{3}{20}n^{-1/d} - \|D_t\|_{\HS} \Big(4\sqrt{\frac{2\log 2d}{n}} + \sqrt{\frac{2\log (2/\delta)}{n}}\Big)\Big]
\end{equation}
We conclude the proof by recalling Lemma \ref{lemma. norm growth rate}, which implies that
\begin{equation*}
\|D_t\|_{\HS} = \|a_t\|_{L^2(\rho_0)} = o(\sqrt{t})
\end{equation*}
\end{proof}

\subsection{Relation between $W_2$ metric and $L^2$ loss}
\label{sec. between W2 and L2}

Recall that GAN, as well as other generative models, adopts the generator representation for its modeled distribution:
\begin{equation*}
P = G\#\N := law(X), \quad X = G(Z), \quad Z \sim \N
\end{equation*}
where $\N$ is some fixed input distribution and $G$ is the generator.

Denote by $L^2(\N;\R^d)$ the space of $L^2(\N)$ functions from $\R^d$ to $\R^d$.
Denote by $\mathcal{P}_{ac}(\R^d)$ the space of absolutely continuous probability measures, and by $\mathcal{P}_{2}(\R^d)$ probability measures with finite second moments.

The $L^2$ regression loss on $G$ induces the $W_2$ metric on $P$.
\begin{proposition}
\label{prop. W2 matching}
Given any target distribution $P_* \in \PS_2(\R^d)$, any input distribution $\N \in \PS_{2,ac}(\R^d)$,
and any $L^2(\N;\R^d)$ function $G$, denote $P = G\#\N$, then
\begin{equation}
\label{W2 matching}
W_2(P,P_*) = \inf_{G_*} \|G-G_*\|_{L^2(\N)}
\end{equation}
where $G_*$ ranges among all $L^2(\N;\R^d)$ functions such that $P_* = G_*\#\N$.
\end{proposition}

\begin{proof}
We have $P \in \mathcal{P}_2(\R^d)$ because
\begin{equation*}
\E_{P}[\|\x\|^2] = \E_{\N}[\|G(\x)\|^2] = \|G\|_{L^2(\N)}^2 < \infty
\end{equation*}
The set $\{G_* \in L^2(\N;\R^d)~|~G_*\#\N=P_*\}$ is nonempty by Theorem 2.12 of \cite{villani2003topics}, so the right hand side of (\ref{W2 matching}) is well-defined.
Furthermore, this term is continuous over $G \in L^2(\N;\R^d)$ by triangle inequality.

First, consider the simple case when $P \in \PS_{2,ac}(\R^d)$.
By Theorem 2.12 of \cite{villani2003topics}, there exists an optimal transport map $h$ from $P$ to $P_*$. Then,
\begin{equation*}
W_2(P,P_*) = \|Id-h\|_{L^2(\N)} = \|G-h\circ G\|_{L^2(\N)}
\end{equation*}
It follows that
\begin{equation*}
W_2(P,P_*) \geq \inf_{G_*\#\N=P_*} \|G-G_*\|_{L^2(\N)}
\end{equation*}
Meanwhile, given any $G_* \in L^2(\N;\R^d)$ ($G_*\#\N = P_*$),
define a joint distribution by $\pi = (G,G_*)\#\N$, which is a coupling between $P$ and $P_*$.
Then,
\begin{equation*}
W_2(P,P_*) \leq \E_{\pi(\x,\x')}[\|\x-\x'\|^2] = \|G-G_*\|_{L^2(\Omega)}
\end{equation*}
Taking infimum over $G_*$, we obtain (\ref{W2 matching}).

Next, for the general case with $P \in \mathcal{P}_2(\R^d)$,
define the random variable $Z \sim P$ and an independent random variable $W$ with unit Gaussian distribution.
Since $\N$ is absolutely continuous, Theorem 16 from Chapter 15 of \cite{royden1988real} implies that the measure space $(\N,\R^d)$ is isomorphic to $[0,1]$ with Lebesgue measure,
which is isomorphic to $\R^d$ with the unit Gaussian distribution.
Thus, we can consider $W$ as a random variable defined on the measure space $(\R^d,\N)$.
For any $\epsilon>0$, define
\begin{equation*}
G_{\epsilon} = G+\epsilon W, \quad P_{\epsilon} = G_{\epsilon}\#\N
\end{equation*}
As $\epsilon \to 0^+$, the map $G_{\epsilon}$ converges to $G$ in $L^2(\N;\R^d)$ and $P_{\epsilon}$ converges to $P$ in $W_2$.
Since $P_{\epsilon} \in \mathcal{P}_{2,ac}(\R^d)$, our preceding result implies that
\begin{equation*}
W_2(P_{\epsilon}, P_*) = \inf_{G_*\#\N=P_*} \|G_{\epsilon}-G_*\|_{L^2(\N)}
\end{equation*}
Taking the limit $\epsilon \to 0^+$, we obtain (\ref{W2 matching}) by continuity.
\end{proof}

{\color{black}
\subsection{Finite-neuron discriminator}
\label{sec. finite neurons}

Recall that in Remark \ref{remark: finite neurons}, the finite-neuron discriminator is defined by
\begin{equation*}
D^{(m)}(\x) = \E_{\rho_0^{(m)}(\w,b)}\big[a(\w,b) \sigma(\w\cdot\x+b)\big], \quad \rho_0^{(m)} = \frac{1}{m} \sum_{j=1}^m \delta_{(\w_j,b_j)}, \quad a(\w_j,b_j) = a_j
\end{equation*}
The WGAN loss (\ref{WGAN + RKHS}) is modified into
\begin{align*}
L_D^{(m)}(a) &= \E_{P_*}[D^{(m)}] - \E_P[D^{(m)}] - \|a\|_{L^2(\rho_0^{(m)})}^2\\
&= \frac{1}{m}\sum_{j=1}^m \int a_j \sigma(\w_j\cdot\x+b_j) d(P_*-P)(\x) - \frac{1}{m}\sum_{j=1}^m a_j^2
\end{align*}
Then, it is straightforward to check that the loss (\ref{MMD loss}) becomes
\begin{equation*}
L^{(m)}(P) = \frac{1}{2}\iint k^{(m)} d(P_*-P)^2, \quad L^{(m,n)}(P) = \frac{1}{2}\iint k^{(m)} d(P_*^{(n)}-P)^2
\end{equation*}
where the kernel is given by
\begin{equation*}
k^{(m)}(\x,\x') = \E_{\rho_0^{(m)}(\w,b)}[\sigma(\w\cdot\x+b)\sigma(\w\cdot\x'+b)] = \frac{1}{m} \sum_{j=1}^m \sigma(\w_j\cdot\x+b_j) \sigma(\w_j\cdot\x'+b_j)
\end{equation*}
Similar to (\ref{MMD gradient descent}), the training dynamics of the population-loss trajectory $p^{(m)}_t$ and the empirical-loss trajectory $p^{(m,n)}_t$ are given by
\begin{align*}
\frac{d}{dt} p^{(m)}_t &= k^{(m)} * (P_*-P^{(m)}_t)\\
\frac{d}{dt} p^{(m,n)}_t &= k^{(m)} * (P_*^{(n)}-P^{(m,n)}_t)
\end{align*}

\begin{lemma}
\label{lemma. empirical op gap}
With probability $1-\delta$ over the sampling of $\rho_0^{(m)}$, the operator norm of $k-k^{(m)}$ over $L^2(\Omega)$ is bounded by
\begin{equation*}
\|k-k^{(m)}\|_{op} \leq \frac{2+\sqrt{\log(4/\delta)}/2}{\sqrt{m}}
\end{equation*}
\end{lemma}
\begin{proof}
For convenience, denote $\tilde{\w}=(\w,b)$ and $\tilde{\x}=(\x,1)$.
Since $(k-k^{(m)})*$ is a symmetric operator,
\begin{align*}
\|k-k^{(m)}\|_{op} &= \sup_{\|h\|_{L^2(\Omega)} \leq 1} \big|\lb (k-k^{(m)})*h, h\rb_{L^2(\Omega)}\big|\\
&= \sup_{\|h\|_{L^2(\Omega)} \leq 1} \Big| \int \Big(\int_{\Omega} h(\x) \sigma(\tilde{\w}\cdot\tilde{\x}) d\x\Big)^2 d(\rho_0-\rho_0^{(m)})(\tilde{\w}) \Big|\\
&= \sup_{f_0 \in \F_0} \Big| \int f_0(\tilde{\w}) d(\rho_0-\rho_0^{(m)})(\tilde{\w}) \Big|
\end{align*}
where we define the function families
\begin{align*}
\F_0 &= \Big\{\tilde{\w} \mapsto \big(\int_{\Omega} h(\x) \sigma(\tilde{\w}\cdot\tilde{\x}) d\x\big)^2 ~\Big|~ \|h\|_{L^2(\Omega)} \leq 1\Big\}\\
\F_1 &= \Big\{\tilde{\w} \mapsto \int_{\Omega} h(\x) \sigma(\tilde{\w}\cdot\tilde{\x}) d\x ~\Big|~ \|h\|_{L^2(\Omega)} \leq 1\Big\}
\end{align*}
Given any set $W = \{\tilde{\w}_j\}_{j=1}^m \subseteq \sprt \rho_0$, the Rademacher complexity is defined as
\begin{equation*}
Rad_m(\F \circ W) = \frac{1}{m} \E_{\mathbf{\xi} \in \{\pm 1\}^m}\Big[\sup_{f\in\F} \sum_{j=1}^m \xi_j f(\tilde{\w}_j)\Big]
\end{equation*}
where $\xi_j$ are i.i.d. random variables with a uniform distribution over $\{\pm 1\}$.

Recall that $\sigma$ is ReLU, $\Omega = [0,1]^d$ and $\sprt \rho_0 \subseteq \{\|\tilde{\w}\|_1 \leq 1\} \subseteq \R^{d+1}$.
It follows that
\begin{equation*}
\sup_{\tilde{\w} \in W} \sup_{f_1 \in \F_1} |f_1(\tilde{\w})| \leq \sup_{\tilde{\w} \in \sprt\rho_0} \sup_{\x\in\Omega} |\sigma(\tilde{\w}\cdot\tilde{\x})| \leq 1
\end{equation*}
Thus the contraction property of $Rad_m$ (Lemma 26.9 \cite{shalev2014understanding}) implies that
\begin{equation*}
Rad_m(\F_0 \circ W) \leq Rad_m(\F_1 \circ W)
\end{equation*}
Meanwhile, by Jensen's inequality
\begin{align*}
Rad_m(\F_1 \circ W) &= \frac{1}{m} \E_{\mathbf{\xi} \in \{\pm 1\}^m}\Big[\sup_{\|h\|_{L^2(\Omega)} \leq 1} \sum_{j=1}^m \xi_j \int_{\Omega} h(\x) \sigma(\tilde{\w}_j \cdot \tilde{\x})\Big]\\
&= \frac{1}{m} \E_{\mathbf{\xi} \in \{\pm 1\}^m}\Big[\sup_{\|h\|_{L^2(\Omega)} \leq 1} \int_{\Omega} h(\x) \sum_{j=1}^m \xi_j \sigma(\tilde{\w}_j \cdot \tilde{\x})\Big]\\
&\leq \frac{1}{m} \E_{\mathbf{\xi} \in \{\pm 1\}^m}\Big[\big\|\sum_{j=1}^m \xi_j \sigma(\tilde{\w}_j \cdot \tilde{\x})\big\|_{L^2(\Omega)}\Big]\\
&\leq \frac{1}{m} \Big(\E_{\mathbf{\xi} \in \{\pm 1\}^m}\Big[\big\|\sum_{j=1}^m \xi_j \sigma(\tilde{\w}_j \cdot \tilde{\x})\big\|_{L^2(\Omega)}^2\Big] \Big)^{1/2}\\
&= \frac{1}{m} \Big(\E_{\mathbf{\xi} \in \{\pm 1\}^m}\Big[ \int_{\Omega} \sum_{j=1}^m\sum_{l=1}^m \xi_j \xi_l \sigma(\tilde{\w}_j \cdot \tilde{\x}) \sigma(\tilde{\w}_l \cdot \tilde{\x}) \Big] \Big)^{1/2}\\
&= \frac{1}{m} \Big(\int_{\Omega} \sum_{j=1}^m \sigma(\tilde{\w}_j \cdot \tilde{\x})^2 \Big)^{1/2}\\
&\leq \frac{1}{m} \Big(\int_{\Omega} \sum_{j=1}^m 1 \Big)^{1/2}\\
&\leq \frac{1}{\sqrt{m}}
\end{align*}
Since the calculations hold for any subset $W$, we have
\begin{equation*}
\sup \big\{ Rad_m(\F_0 \circ W) ~\big|~ W=\{\tilde{\w}_j\}_{j=1}^m \subseteq \sprt\rho_0 \big\} \leq \frac{1}{\sqrt{m}}
\end{equation*}

Note that for all $f_0 \in \F_0$, the function $f_0$ ranges in $[0,1]$.
Then, Theorem 26.5 of \cite{shalev2014understanding} implies that with probability $1-\delta$ over the sampling of $\rho_0^{(m)}$,
\begin{align*}
&\quad \sup_{f_0 \in \F_0} \int f_0(\tilde{\w}) d(\rho_0-\rho_0^{(m)})(\tilde{\w})\\
&\leq 2 \E\big[Rad_m(\F_0 \circ W) ~\big|~ (\tilde{\w}_1, \dots \tilde{\w}_m) \sim \rho_0^m\big] + \sqrt{\frac{\log(2/\delta)}{2m}}\\
&\leq 2 \sup \big\{ Rad_m(\F_0 \circ W) ~\big|~ W=(\tilde{\w}_1, \dots \tilde{\w}_m) \subseteq \sprt\rho_0 \big\} + \sqrt{\frac{\log(2/\delta)}{2m}}\\
&\leq \frac{2+\sqrt{\log(2/\delta)/2}}{\sqrt{m}}
\end{align*}
Applying the same argument to the function family $\{-f_0|f_0\in\F_0\}$, we obtain that with probability $1-2\delta$,
\begin{equation*}
\sup_{f_0 \in \F_0} \Big|\int f_0(\tilde{\w}) d(\rho_0-\rho_0^{(m)})\Big| \leq \frac{2+\sqrt{\log(2/\delta)/2}}{\sqrt{m}}
\end{equation*}
which concludes the proof.
\end{proof}

It remains to compare the training dynamics of $p_t$ and $p^{(m)}_t$:
\begin{align*}
\frac{d}{dt} \|p_t-p^{(m)}_t\|_{L^2(\Omega)} &= \big\lb \frac{p_t-p^{(m)}_t}{\|p_t-p^{(m)}_t\|}, ~k*(p_*-p_t) - k^{(m)} * (p_*-p^{(m)}_t) \big\rb\\
&= \big\lb \frac{p_t-p^{(m)}_t}{\|p_t-p^{(m)}_t\|}, ~(k-k^{(m)}) * (p_*-p_t) - k^{(m)}*(p_t-p^{(m)}_t)\big\rb\\
&\leq \big\lb \frac{p_t-p^{(m)}_t}{\|p_t-p^{(m)}_t\|}, ~(k-k^{(m)}) * (p_*-p_t) \big\rb\\
&\leq \|k-k^{(m)}\|_{op} \|p_*-p_t\|_{L^2(\Omega)}\\
&\leq \|k-k^{(m)}\|_{op} \frac{\|p_*-p_0\|_{\HS}}{\sqrt{t}}
\end{align*}
with the last step given by Lemma \ref{lemma. training error}.
Then, Lemma \ref{lemma. empirical op gap} implies that with probability $1-\delta$,
\begin{align*}
\|p_t-p^{(m)}_t\|_{L^2(\Omega)} &\leq 2\sqrt{t} \|p_*-p_0\|_{\HS} \|k-k^{(m)}\|_{op}\\
&\leq \|p_*-p_0\|_{\HS} \frac{4+\sqrt{2\log(4/\delta)}}{\sqrt{m}} \sqrt{t}
\end{align*}

Moreover, it is straightforward to check that the proof of inequality (\ref{L2 generalization gap}) continues to hold if the kernel $k$ is replaced by $k^{(m)}$, so that (\ref{L2 generalization gap}) can be directly modified into
\begin{equation*}
\|p^{(m)}_t-p^{(m,n)}_t\|_{L^2} \leq \frac{4\sqrt{2\log 2d} + \sqrt{2\log (2/\delta)}}{\sqrt{n}}t
\end{equation*}

Hence, we can conclude with Lemma \ref{lemma. training error} that, with probability $1-2\delta$,
\begin{align*}
\|p_*-p^{(m,n)}_t\|_{L^2(\Omega)} &\leq \|p_*-p_t\|_{L^2} + \|p_t-p^{(m)}_t\|_{L^2} + \|p^{(m)}_t-p^{(m,n)}_t\|_{L^2}\\
&\leq \frac{\|p_*-p_0\|_{\HS}}{\sqrt{t}} + \frac{\|p_*-p_0\|_{\HS} \big[4+\sqrt{2\log(4/\delta)}\big]}{\sqrt{m}} \sqrt{t} + \frac{4\sqrt{2\log 2d} + \sqrt{2\log (2/\delta)}}{\sqrt{n}}t
\end{align*}
Then, Remark \ref{remark: finite neurons} follows from Lemma \ref{lemma. W2-L2 bound}.

\subsection{Gradient analysis for the generator}
This subsection provides the details of the calculation in Section \ref{sec. the generator}.

First, to derive the formula
\begin{equation*}
\frac{\delta L}{\delta G} = \nabla_{\x} \frac{\delta L}{\delta P} \circ G
\end{equation*}
one simply note that for any perturbation $h \in L^2(\N;\R^d)$
\begin{align*}
\Big\lb h, \frac{\delta L}{\delta G} \Big\rb_{L^2(\N)} &= \lim_{\ep \to 0} \frac{L((G+\ep h)\#\N) - L(G\#\N)}{\ep}\\
&= \lim_{\ep \to 0} \int \frac{\delta L}{\delta P}\Big|_{P=G\#\N} ~\text{d} \frac{(G+\epsilon h)\#\N - G\#\N}{\epsilon}\\
&= \lim_{\ep \to 0} \int \frac{1}{\ep} \Big[ \frac{\delta L}{\delta P}(G+\ep h) - \frac{\delta L}{\delta P}(G) \Big] \text{d}\N\\
&= \int \nabla_{\x} \frac{\delta L}{\delta P}(G) \cdot h ~\text{d}\N\\
&= \Big\lb h, \nabla_{\x} \frac{\delta L}{\delta P} \circ G \Big\rb_{L^2(\N)}
\end{align*}

Next, we try to bound the norm of the following term
\begin{align*}
\frac{\delta L}{\delta G}-\frac{\delta L^{(n)}}{\delta G} &= \nabla_{\x} \Big(\frac{\delta L}{\delta P} - \frac{\delta L^{(n)}}{\delta P}\Big) \circ G = \nabla_1 k * (P_*-P_*^{(n)}) \circ G
\end{align*}
where $\nabla_1$ means taking the gradient in the first entry of $k$.
Note that
\begin{align*}
\|\nabla_1 k * (P_*-P_*^{(n)}) \circ G\|_{L^2(\N)} &= \|\nabla_1 k * (P_*-P_*^{(n)})\|_{L^2(P)}\\
&= \Big\| \E_{\rho_0(\w,b)} \big[ \nabla_{\x} \sigma(\w\cdot\x+b) \int \sigma(\w\cdot\x'+b) d(P_*-P_*^{(n)})\big] \Big\|_{L^2(P)}\\
&\leq \sup_{\w,b \in \sprt \rho_0} \Big| \|\w\|_2 \|\sigma\|_{Lip} \int \sigma(\w\cdot\x'+b) d(P_*-P_*^{(n)}) \Big|\\
&\leq \sup_{\|\w\|_1 + |b| \leq 1} \Big| \int \sigma(\w\cdot\x'+b) d(P_*-P_*^{(n)}) \Big|
\end{align*}
Then, Lemma \ref{lemma. RKHS Monte Carlo rate} implies that with probability $1-\delta$,
\begin{equation*}
\Big\| \frac{\delta L}{\delta G}-\frac{\delta L^{(n)}}{\delta G} \Big\|_{L^2(\N)} \leq \frac{4\sqrt{2\log 2d} + \sqrt{2\log (2/\delta)}}{\sqrt{n}}
\end{equation*}
}

\section{Discussion}
\label{sec. discussion}

Let us conclude with some of the insights obtained in this paper:
\begin{itemize}
\item Good generalization is achievable in high dimensions by early stopping and the error estimate escapes from the curse of dimensionality,
whereas in the long term ($t\to\infty$), the trained distribution slowly deteriorates to the empirical distribution and exhibits memorization.
This is an implicit regularization result, such that the progress toward $P_*$ and the deterioration due to $P_*-P_*^{(n)}$ occur on two time scales.

\item The mechanism for generalization is the dimension-independent complexity of the function representation of the discriminator, which in our setting is the Rademacher complexity of random feature functions.
It renders the loss landscape insensitive to the sampling error $P_*-P_*^{(n)}$, and thereby delays the onset of memorization.
Our hardness of learning result for WGAN also follows from this small complexity.

\item The regularization $R(D)$ is crucial for establishing the convergence of training.
As demonstrated by the proof of Theorem \ref{thm. one-time-scale generalization error}, the role of $R(D)$ is to introduce ``friction" into the min-max training of $P_t$ and $D_t$ and dampens their oscillatory dynamics.
Furthermore, Theorem \ref{thm. generalization error} demonstrates that regularizations on parameters may perform better than regularizations on function value.
The former imposes a tighter control on the complexity of the discriminator, and thus the growth of the generalization gap is slower.
Beyond the RKHS norm (\ref{RKHS norm}), one can also consider the spectral norm \cite{miyato2018spectral} or the Barron norm and flow-induced norm \cite{e2021barron}.
\end{itemize}

Finally, we discuss several possible directions for future research.

{\color{black}
\subsection{The generator}
\label{sec. the generator}

So far this paper has focused on analyzing the discriminator with the generator omitted.
Here we show that part of our proof of generalization can be extended to the ordinary GAN with a generator.

One key step in our argument is the comparison of the landscapes of the population loss $L$ and the empirical loss $L^{(n)}$:
For any finite measure $P$, with high probability, we can compare the variational derivatives
\begin{equation*}
\Big\|\frac{\delta L}{\delta P} - \frac{\delta L^{(n)}}{\delta P}\Big\|_{L^2(\Omega)} = \|k*(P_*-P_*^{(n)})\| = O(n^{-1/2})
\end{equation*}
(See the proof of Proposition \ref{prop. generalization gap} for more details.)
It is this closeness between the two derivatives that eventually leads to the estimate of the generalization gap.

For the ordinary GAN, the loss becomes $L(P), P=G\#\N$, where $\N$ is some input distribution such as the unit Gaussian.
The variational gradient of the generator $G$ over $L^2(\N;\R^d)$ is simply
\begin{equation*}
\frac{\delta L}{\delta G} = \nabla_{\x} \frac{\delta L}{\delta P} \circ G
\end{equation*}
It follows that the population loss and empirical loss landscapes differ by
\begin{align*}
\Big\|\frac{\delta L}{\delta G}-\frac{\delta L^{(n)}}{\delta G}\Big\|_{L^2(\N)} &= \Big\|\nabla_{\x} \Big(\frac{\delta L}{\delta P} - \frac{\delta L^{(n)}}{\delta P}\Big) \circ G\Big\|_{L^2(\N)}\\
&= \|\nabla_1 k * (P_*-P_*^{(n)}) \|_{L^2(P)}\\
&= O(n^{-1/2})
\end{align*}
(See Section 6.5 for more details.)
Then, one can conclude heuristically that the generalization gap for the ordinary GAN should also scale as $O(t/\sqrt{n})$.
Nevertheless, the gap between this heuristic argument and a rigorous proof is that the loss landscape for $G$ is generally nonconvex, so the proof routines in Sections 6.1.1 and 6.1.2 need to be modified.

\subsection{Wasserstein gradient flow}
\label{sec. Wasserstein gradient flow}

As discussed in Section \ref{sec. distribution representation}, the distribution $P$ is modeled as a density function in order to remove any parametrization in the generator.
Its training depends on the linear topology of $\PS(\Omega)$, such that the density function $p_t$ is updated vertically (\ref{MMD gradient descent}, \ref{MMD projected gradient descent}).
However, to better capture the training of GAN (and generative models in general), one can try to use horizontal updates:
Consider the distribution $P_t = G_t \# \N$, where $G_t$ is the generator during training, then the trajectory $P_{[0,T]}$ can be seen as a random smooth path
\begin{equation*}
P_{[0,T]} = law(\x_{[0,T]}), \quad \x_{[0,T]} = G_{[0,T]}(\z) \in C^1([0,T]\to\R^d), \quad \z \sim \N
\end{equation*}
In particular, $P_t$ satisfies the conservation of local mass, unlike the vertical updates that teleport mass.

One natural way to perform horizontal updates without any parametrization in $G$ is to use the Wasserstein gradient flow
\begin{equation*}
\partial_t P_t = \nabla\cdot\Big(P_t\nabla\frac{\delta L(P_t)}{\delta P}\Big) = \nabla \cdot\big(P_t \nabla k*(P_t-P_*)\big)
\end{equation*}
One can try to bound its generalization gap as in Proposition \ref{prop. generalization gap}.

\begin{conjecture}
Let $P_t,P_t^{(n)}$ be the population-loss trajectory and empirical-loss trajectory.
Then, for any $\delta \in (0,1)$, with probability $1-\delta$ over the sampling of $P_*^{(n)}$,
\begin{equation*}
W_2(P_t,P_t^{(n)}) \lesssim \frac{\sqrt{\log d}+\sqrt{\log 1/\delta}}{\sqrt{n}} t
\end{equation*}
\end{conjecture}

The difficulty for horizontal updates is that one typically needs geodesic convexity in $W_2$ space, which is not known to hold for any of the GAN losses.
For instance, the geodesic nonconvexity of the MMD metric (\ref{MMD loss}) has been analyzed in \cite{arbel2019maximum}.
}

\subsection{Sophisticated discriminators}
Instead of the WGAN loss with random feature functions, one can study more general set-ups with diverse losses and discriminators.

For the function representation of $D$, ideally one would like to consider any family $\D$ whose Rademacher complexity scales as $O(n^{-1/2})$.
Examples include 2-layer networks with bounded Barron norm \cite{e2021barron}, deep residual networks with bounded flow-induced norm \cite{e2021barron} and multilayer networks with bounded path norm \cite{e2020banach}.

For the training loss, one can consider the classical GAN loss \cite{goodfellow2014generative}, the $f$-GAN loss \cite{nowozin2016f}, the energy-based GAN loss \cite{zhao2016energy}, the least-square GAN loss \cite{mao2018effectiveness}, and any other losses defined as the dual over a discriminator family $\D$.

{\color{black}
Then, one can try to extend the generalization gap from Proposition \ref{prop. generalization gap} to these settings.
The added difficulty, however, is that one no longer has a close-form formula for the dynamics of $p_t$ as in (\ref{MMD gradient descent}), making the comparison of the trajectories over the population loss and empirical loss less explicit. 
}

\subsection{Slower deterioration}

A shortcoming of Proposition \ref{prop. slow deterioration} is that its proof (in particular, inequality (\ref{Linf gap via norm})) does not utilize the fact that the discriminator $D_t$ during training has bounded Lipschitz norm
\begin{equation*}
\forall t > 0, \quad \|D_t\|_{\Lip} \lesssim 1
\end{equation*}
It seems reasonable that a more refined analysis would lead to a much stronger lower bound.
For instance, one might conjecture that with probability $1-\delta$,
\begin{equation*}
\inf_D\big\{\|\nabla D-\nabla D_*\|_{L^2(\Omega)} ~\big|~ \|D\|_{\HS} \leq R, ~\|D\|_{\Lip} \leq 1\big\} \gtrsim \delta^{\frac{1}{nd}} R^{-\frac{2}{d-2}} n^{-\frac{1}{d}} - \frac{\sqrt{\log 2d} + \sqrt{\log 2/\delta}}{\sqrt{n}}
\end{equation*}
The exponent $\frac{2}{d-2}$ comes from Corollary 3.4 of \cite{e2020kolmogorov}, and we estimate $\|\nabla D-\nabla D_*\|$ because it is the gradient field $\nabla D_t$ that drives the distribution during GAN training.

Then, Lemma \ref{lemma. norm growth rate} implies that it would take at least $n^{\Omega(d^2)}$ amount of time for $D_t$ to learn $D_*$.

\section*{Conflict of Interest}
On behalf of all authors, the corresponding author states that there is no conflict of interest.

\bibliography{main}

\begin{thebibliography}{10}

\bibitem{ambrosio2008gradient}
{\sc Ambrosio, L., Gigli, N., and Savar{\'e}, G.}
\newblock {\em Gradient flows: in metric spaces and in the space of probability
  measures}.
\newblock Springer Science \& Business Media, 2008.

\bibitem{arbel2019maximum}
{\sc Arbel, M., Korba, A., Salim, A., and Gretton, A.}
\newblock Maximum mean discrepancy gradient flow.
\newblock {\em arXiv preprint arXiv:1906.04370\/} (2019).

\bibitem{arjovsky2017wasserstein}
{\sc Arjovsky, M., Chintala, S., and Bottou, L.}
\newblock Wasserstein {GAN}.
\newblock {\em arXiv preprint arXiv:1701.07875\/} (2017).

\bibitem{arora2017generalization}
{\sc Arora, S., Ge, R., Liang, Y., Ma, T., and Zhang, Y.}
\newblock Generalization and equilibrium in generative adversarial nets
  ({GAN}s).
\newblock {\em arXiv preprint arXiv:1703.00573\/} (2017).

\bibitem{arora2018GAN}
{\sc Arora, S., Risteski, A., and Zhang, Y.}
\newblock Do {GANs} learn the distribution? {S}ome theory and empirics.
\newblock In {\em International Conference on Learning Representations\/}
  (2018).

\bibitem{ba2016layer}
{\sc Ba, J.~L., Kiros, J.~R., and Hinton, G.~E.}
\newblock Layer normalization.
\newblock {\em arXiv preprint arXiv:1607.06450\/} (2016).

\bibitem{bai2019approximability}
{\sc Bai, Y., Ma, T., and Risteski, A.}
\newblock Approximability of discriminators implies diversity in {GAN}s, 2019.

\bibitem{balaji2021understanding}
{\sc Balaji, Y., Sajedi, M., Kalibhat, N.~M., Ding, M., St{\"o}ger, D.,
  Soltanolkotabi, M., and Feizi, S.}
\newblock Understanding overparameterization in generative adversarial
  networks.
\newblock {\em arXiv preprint arXiv:2104.05605\/} (2021).

\bibitem{borkar1997timescale}
{\sc Borkar, V.~S.}
\newblock Stochastic approximation with two time scales.
\newblock {\em Systems \& Control Letters 29}, 5 (1997), 291--294.

\bibitem{chavdarova2018sgan}
{\sc Chavdarova, T., and Fleuret, F.}
\newblock {SGAN}: An alternative training of generative adversarial networks.
\newblock In {\em Proceedings of the IEEE Conference on Computer Vision and
  Pattern Recognition\/} (2018), pp.~9407--9415.

\bibitem{che2016mode}
{\sc Che, T., Li, Y., Jacob, A., Bengio, Y., and Li, W.}
\newblock Mode regularized generative adversarial networks.
\newblock {\em arXiv preprint arXiv:1612.02136\/} (2016).

\bibitem{dobric1995asymptotics}
{\sc Dobri{\'c}, V., and Yukich, J.~E.}
\newblock Asymptotics for transportation cost in high dimensions.
\newblock {\em Journal of Theoretical Probability 8}, 1 (1995), 97--118.

\bibitem{e2019residual}
{\sc E, W., Ma, C., and Wang, Q.}
\newblock A priori estimates of the population risk for residual networks.
\newblock {\em arXiv preprint arXiv:1903.02154 1}, 7 (2019).

\bibitem{e2020NNML}
{\sc E, W., Ma, C., Wojtowytsch, S., and Wu, L.}
\newblock Towards a mathematical understanding of neural network-based machine
  learning: what we know and what we don't, 2020.

\bibitem{e2018priori}
{\sc E, W., Ma, C., and Wu, L.}
\newblock A priori estimates for two-layer neural networks.
\newblock {\em arXiv preprint arXiv:1810.06397\/} (2018).

\bibitem{e2019min}
{\sc E, W., Ma, C., and Wu, L.}
\newblock On the generalization properties of minimum-norm solutions for
  over-parameterized neural network models.
\newblock {\em arXiv preprint arXiv:1912.06987\/} (2019).

\bibitem{e2020machine}
{\sc E, W., Ma, C., and Wu, L.}
\newblock Machine learning from a continuous viewpoint, i.
\newblock {\em Science China Mathematics 63}, 11 (2020), 2233--2266.

\bibitem{e2021barron}
{\sc E, W., Ma, C., and Wu, L.}
\newblock The {Barron} space and the flow-induced function spaces for neural
  network models.
\newblock {\em Constructive Approximation\/} (2021), 1--38.

\bibitem{e2020kolmogorov}
{\sc E, W., and Wojtowytsch, S.}
\newblock Kolmogorov width decay and poor approximators in machine learning:
  Shallow neural networks, random feature models and neural tangent kernels.
\newblock {\em arXiv preprint arXiv:2005.10807\/} (2020).

\bibitem{e2020banach}
{\sc E, W., and Wojtowytsch, S.}
\newblock On the banach spaces associated with multi-layer {ReLU} networks:
  Function representation, approximation theory and gradient descent dynamics.
\newblock {\em arXiv preprint arXiv:2007.15623\/} (2020).

\bibitem{feizi2020LQG}
{\sc {Feizi}, S., {Farnia}, F., {Ginart}, T., and {Tse}, D.}
\newblock Understanding {GANs} in the {LQG} setting: Formulation,
  generalization and stability.
\newblock {\em IEEE Journal on Selected Areas in Information Theory 1}, 1
  (2020), 304--311.

\bibitem{goodfellow2014generative}
{\sc Goodfellow, I., Pouget-Abadie, J., Mirza, M., Xu, B., Warde-Farley, D.,
  Ozair, S., Courville, A., and Bengio, Y.}
\newblock Generative adversarial nets.
\newblock In {\em Advances in neural information processing systems\/} (2014),
  pp.~2672--2680.

\bibitem{gretton2021kernel}
{\sc Gretton, A., Borgwardt, K.~M., Rasch, M.~J., Sch{{\"o}}lkopf, B., and
  Smola, A.}
\newblock A kernel two-sample test.
\newblock {\em Journal of Machine Learning Research 13}, 25 (2012), 723--773.

\bibitem{gulrajani2017improved}
{\sc Gulrajani, I., Ahmed, F., Arjovsky, M., Dumoulin, V., and Courville, A.}
\newblock Improved training of {Wasserstein} {GAN}s, 2017.

\bibitem{gulrajani2020towards}
{\sc Gulrajani, I., Raffel, C., and Metz, L.}
\newblock Towards {GAN} benchmarks which require generalization.
\newblock {\em arXiv preprint arXiv:2001.03653\/} (2020).

\bibitem{heusel2017gans}
{\sc Heusel, M., Ramsauer, H., Unterthiner, T., Nessler, B., and Hochreiter,
  S.}
\newblock {GANs} trained by a two time-scale update rule converge to a local
  {N}ash equilibrium.
\newblock In {\em Advances in neural information processing systems\/} (2017),
  pp.~6626--6637.

\bibitem{hornik1991approximation}
{\sc Hornik, K.}
\newblock Approximation capabilities of multilayer feedforward networks.
\newblock {\em Neural Networks 4}, 2 (1991), 251–257.

\bibitem{ioffe2015batch}
{\sc Ioffe, S., and Szegedy, C.}
\newblock Batch normalization: Accelerating deep network training by reducing
  internal covariate shift.
\newblock {\em arXiv preprint arXiv:1502.03167\/} (2015).

\bibitem{isola2017image}
{\sc Isola, P., Zhu, J.-Y., Zhou, T., and Efros, A.~A.}
\newblock Image-to-image translation with conditional adversarial networks.
\newblock In {\em Proceedings of the IEEE conference on computer vision and
  pattern recognition\/} (2017), pp.~1125--1134.

\bibitem{jiang2021transgan}
{\sc Jiang, Y., Chang, S., and Wang, Z.}
\newblock {TransGAN}: Two transformers can make one strong {GAN}.
\newblock {\em arXiv preprint arXiv:2102.07074\/} (2021).

\bibitem{karras2019style}
{\sc Karras, T., Laine, S., and Aila, T.}
\newblock A style-based generator architecture for generative adversarial
  networks.
\newblock In {\em Proceedings of the IEEE/CVF Conference on Computer Vision and
  Pattern Recognition\/} (2019), pp.~4401--4410.

\bibitem{kingma2013auto}
{\sc Kingma, D.~P., and Welling, M.}
\newblock Auto-encoding variational bayes.
\newblock {\em arXiv preprint arXiv:1312.6114\/} (2013).

\bibitem{kodali2017convergence}
{\sc Kodali, N., Abernethy, J., Hays, J., and Kira, Z.}
\newblock On convergence and stability of {GANs}.
\newblock {\em arXiv preprint arXiv:1705.07215\/} (2017).

\bibitem{krogh1992simple}
{\sc Krogh, A., and Hertz, J.~A.}
\newblock A simple weight decay can improve generalization.
\newblock In {\em Advances in neural information processing systems\/} (1992),
  pp.~950--957.

\bibitem{lei2020sgd}
{\sc Lei, Q., Lee, J.~D., Dimakis, A.~G., and Daskalakis, C.}
\newblock Sgd learns one-layer networks in wgans, 2020.

\bibitem{liang2021unpaired}
{\sc Liang, Y., Lee, D., Li, Y., and Shin, B.-S.}
\newblock Unpaired medical image colorization using generative adversarial
  network.
\newblock {\em Multimedia Tools and Applications\/} (2021), 1--15.

\bibitem{lin2020gradient}
{\sc Lin, T., Jin, C., and Jordan, M.}
\newblock On gradient descent ascent for nonconvex-concave minimax problems.
\newblock In {\em International Conference on Machine Learning\/} (2020), PMLR,
  pp.~6083--6093.

\bibitem{mao2018effectiveness}
{\sc Mao, X., Li, Q., Xie, H., Lau, R.~Y., Wang, Z., and Smolley, S.~P.}
\newblock On the effectiveness of least squares generative adversarial
  networks.
\newblock {\em IEEE transactions on pattern analysis and machine intelligence
  41}, 12 (2018), 2947--2960.

\bibitem{mao2020material}
{\sc Mao, Y., He, Q., and Zhao, X.}
\newblock Designing complex architectured materials with generative adversarial
  networks.
\newblock {\em Science Advances 6}, 17 (2020).

\bibitem{mescheder2018training}
{\sc Mescheder, L., Geiger, A., and Nowozin, S.}
\newblock Which training methods for gans do actually converge?
\newblock In {\em International conference on machine learning\/} (2018), PMLR,
  pp.~3481--3490.

\bibitem{miyato2018spectral}
{\sc Miyato, T., Kataoka, T., Koyama, M., and Yoshida, Y.}
\newblock Spectral normalization for generative adversarial networks.
\newblock {\em arXiv preprint arXiv:1802.05957\/} (2018).

\bibitem{mustafa2019cosmogan}
{\sc Mustafa, M., Bard, D., Bhimji, W., Luki{\'c}, Z., Al-Rfou, R., and
  Kratochvil, J.~M.}
\newblock {CosmoGAN}: creating high-fidelity weak lensing convergence maps
  using generative adversarial networks.
\newblock {\em Computational Astrophysics and Cosmology 6}, 1 (2019), 1.

\bibitem{nagarajan2017gradient}
{\sc Nagarajan, V., and Kolter, J.~Z.}
\newblock Gradient descent gan optimization is locally stable.
\newblock {\em arXiv preprint arXiv:1706.04156\/} (2017).

\bibitem{nagarajan2018memorization}
{\sc Nagarajan, V., Raffel, C., and Goodfellow, I.}
\newblock Theoretical insights into memorization in {GANs}.
\newblock In {\em Neural Information Processing Systems Workshop\/} (2018).

\bibitem{nowozin2016f}
{\sc Nowozin, S., Cseke, B., and Tomioka, R.}
\newblock $f$-{GAN}: Training generative neural samplers using variational
  divergence minimization.
\newblock In {\em Advances in neural information processing systems\/} (2016),
  pp.~271--279.

\bibitem{petzka2018regularization}
{\sc Petzka, H., Fischer, A., and Lukovnicov, D.}
\newblock On the regularization of {Wasserstein GANs}, 2018.

\bibitem{prykhodko2019molecular}
{\sc Prykhodko, O., Johansson, S.~V., Kotsias, P.-C., Arús-Pous, J., Bjerrum,
  E.~J., Engkvist, O., and Chen, H.}
\newblock A de novo molecular generation method using latent vector based
  generative adversarial network.
\newblock {\em Journal of Cheminformatics 11}, 74 (Dec 2019), 1–11.

\bibitem{radford2015unsupervised}
{\sc Radford, A., Metz, L., and Chintala, S.}
\newblock Unsupervised representation learning with deep convolutional
  generative adversarial networks.
\newblock {\em arXiv preprint arXiv:1511.06434\/} (2015).

\bibitem{rahimi2008uniform}
{\sc Rahimi, A., and Recht, B.}
\newblock Uniform approximation of functions with random bases.
\newblock In {\em 2008 46th Annual Allerton Conference on Communication,
  Control, and Computing\/} (2008), IEEE, pp.~555--561.

\bibitem{royden1988real}
{\sc Royden, H.~L.}
\newblock {\em Real analysis}, 3~ed.
\newblock Collier Macmillan, London, 1988.

\bibitem{saxena2021generative}
{\sc Saxena, D., and Cao, J.}
\newblock Generative adversarial networks ({GANs}) challenges, solutions, and
  future directions.
\newblock {\em ACM Computing Surveys (CSUR) 54}, 3 (2021), 1--42.

\bibitem{shah2018solving}
{\sc Shah, V., and Hegde, C.}
\newblock Solving linear inverse problems using {GAN} priors: An algorithm with
  provable guarantees.
\newblock In {\em 2018 IEEE international conference on acoustics, speech and
  signal processing (ICASSP)\/} (2018), IEEE, pp.~4609--4613.

\bibitem{shalev2014understanding}
{\sc Shalev-Shwartz, S., and Ben-David, S.}
\newblock {\em Understanding machine learning: From theory to algorithms}.
\newblock Cambridge university press, 2014.

\bibitem{singh2018minimax}
{\sc Singh, S., and P{\'o}czos, B.}
\newblock Minimax distribution estimation in {Wasserstein} distance.
\newblock {\em arXiv preprint arXiv:1802.08855\/} (2018).

\bibitem{sun2018RFM}
{\sc Sun, Y., Gilbert, A., and Tewari, A.}
\newblock On the approximation properties of random {ReLU} features.
\newblock {\em arXiv preprint arXiv:1810.04374\/} (2018).

\bibitem{tabak2010density}
{\sc Tabak, E.~G., Vanden-Eijnden, E., et~al.}
\newblock Density estimation by dual ascent of the log-likelihood.
\newblock {\em Communications in Mathematical Sciences 8}, 1 (2010), 217--233.

\bibitem{villani2003topics}
{\sc Villani, C.}
\newblock {\em Topics in optimal transportation}.
\newblock No.~58 in Graduate Studies in Mathematics. American Mathematical
  Soc., 2003.

\bibitem{weed2017sharp}
{\sc Weed, J., and Bach, F.}
\newblock Sharp asymptotic and finite-sample rates of convergence of empirical
  measures in {Wasserstein} distance.
\newblock {\em arXiv preprint arXiv:1707.00087\/} (2017).

\bibitem{wojtowytsch2020convergence}
{\sc Wojtowytsch, S.}
\newblock On the convergence of gradient descent training for two-layer
  {ReLU}-networks in the mean field regime.
\newblock {\em arXiv preprint arXiv:2005.13530\/} (2020).

\bibitem{wu2019gp}
{\sc Wu, H., Zheng, S., Zhang, J., and Huang, K.}
\newblock {GP-GAN}: Towards realistic high-resolution image blending.
\newblock In {\em Proceedings of the 27th ACM international conference on
  multimedia\/} (2019), pp.~2487--2495.

\bibitem{wu2019onelayer}
{\sc Wu, S., Dimakis, A.~G., and Sanghavi, S.}
\newblock Learning distributions generated by one-layer {ReLU} networks.
\newblock In {\em Advances in Neural Information Processing Systems 32}. Curran
  Associates, Inc., 2019, pp.~8107--8117.

\bibitem{xu2020understanding}
{\sc Xu, K., Li, C., Zhu, J., and Zhang, B.}
\newblock Understanding and stabilizing {GANs’} training dynamics using
  control theory.
\newblock In {\em International Conference on Machine Learning\/} (2020), PMLR,
  pp.~10566--10575.

\bibitem{yang2020generalization}
{\sc Yang, H., and E, W.}
\newblock Generalization and memorization: The bias potential model, 2020.

\bibitem{yazici2020empirical}
{\sc Yazici, Y., Foo, C.-S., Winkler, S., Yap, K.-H., and Chandrasekhar, V.}
\newblock Empirical analysis of overfitting and mode drop in gan training.
\newblock In {\em 2020 IEEE International Conference on Image Processing
  (ICIP)\/} (2020), IEEE, pp.~1651--1655.

\bibitem{zhang2017discrimination}
{\sc Zhang, P., Liu, Q., Zhou, D., Xu, T., and He, X.}
\newblock On the discrimination-generalization tradeoff in {GAN}s.
\newblock {\em arXiv preprint arXiv:1711.02771\/} (2017).

\bibitem{zhao2016energy}
{\sc Zhao, J., Mathieu, M., and LeCun, Y.}
\newblock Energy-based generative adversarial network.
\newblock {\em arXiv preprint arXiv:1609.03126\/} (2016).

\end{thebibliography}
\bibliographystyle{acm}

\end{document}